\documentclass{article}

 \usepackage[preprint]{neurips_2026}

\usepackage[normalem]{ulem}
\usepackage{xcolor}

\usepackage[utf8]{inputenc} 
\usepackage[T1]{fontenc}    
\usepackage{hyperref}       
\usepackage{url}            
\usepackage{booktabs}       
\usepackage{amsfonts}       
\usepackage{nicefrac}       
\usepackage{microtype}      
\usepackage{xcolor}         

\usepackage{natbib}
\usepackage{algorithm}
\usepackage{algpseudocode}
\usepackage{paralist}
\usepackage{multirow}

\usepackage{wrapfig}

\usepackage{url}
\usepackage{color}
\usepackage{makeidx}  
\usepackage{xspace}
\usepackage{epstopdf}
\usepackage{mathrsfs}
\usepackage{enumerate}
\usepackage{graphicx,epsfig}
\usepackage{bm}
\usepackage{bbm}
\usepackage{upgreek}
\usepackage{caption}
\usepackage{subcaption}

\usepackage{algorithm}
\usepackage{algpseudocode}  
\usepackage{enumitem}

\usepackage{amsmath}
\usepackage{amssymb}
\usepackage{mathtools}
\usepackage{amsthm}
\usepackage{float}

\usepackage{cleveref}

\title{Dual-Agent Co-Training for Health Coaching via Implicit Adversarial Preference Optimization}

\author{%
  Da Long\\
  Kahlert School of Computing \\
  University of Utah \\
  \texttt{da.long@utah.edu} \\
  \And
  Lingyi Fu \\
  Department of Health and Kinesiology \\
  University of Utah \\
  \texttt{Lingyi.Fu@utah.edu} \\
  \And
  Diya Michelle Rao \\
  Department of Health and Kinesiology \\
  University of Utah \\
  \texttt{diya.rao@utah.edu} \\
  \AND
  Jasmine Ruales Carrera \\
  Department of Health and Kinesiology \\
  University of Utah \\
  \texttt{u1300721@utah.edu} \\
  \And
  Yang Bai \\
  Department of Health and Kinesiology \\
  University of Utah \\
  \texttt{Yang.Bai@utah.edu} \\
  \And
  Shandian Zhe\thanks{Correspondence to: Shandian Zhe
<zhe@cs.utah.edu>.}  \\
  Kahlert School of Computing \\
  University of Utah \\
  \texttt{zhe@cs.utah.edu} \\
}

\theoremstyle{plain}
\newtheorem{theorem}{Theorem}[section]
\newtheorem{proposition}[theorem]{Proposition}
\newtheorem{lemma}[theorem]{Lemma}

\theoremstyle{definition}
\newtheorem{definition}[theorem]{Definition}

\theoremstyle{remark}

\newcommand{\ours}{DACT\xspace}

\begin{document}

\maketitle

\newcommand{\var}{{\rm var}}
\newcommand{\vtrans}[2]{{#1}^{(#2)}}
\newcommand{\kron}{\otimes}
\newcommand{\schur}[2]{({#1} | {#2})}
\newcommand{\schurdet}[2]{\left| ({#1} | {#2}) \right|}
\newcommand{\had}{\circ}
\newcommand{\diag}{{\rm diag}}
\newcommand{\invdiag}{\diag^{-1}}
\newcommand{\rank}{{\rm rank}}
 \newcommand{\expt}[1]{\langle #1 \rangle}
\newcommand{\nullsp}{{\rm null}}
\newcommand{\tr}{{\rm tr}}
\renewcommand{\vec}{{\rm vec}}
\newcommand{\vech}{{\rm vech}}
\renewcommand{\det}[1]{\left| #1 \right|}
\newcommand{\pdet}[1]{\left| #1 \right|_{+}}
\newcommand{\pinv}[1]{#1^{+}}
\newcommand{\erf}{{\rm erf}}
\newcommand{\hypergeom}[2]{{}_{#1}F_{#2}}
\newcommand{\mcal}[1]{\mathcal{#1}}
\newcommand{\bepsilon}{\boldsymbol{\epsilon}}
\newcommand{\brho}{\boldsymbol{\rho}}
\renewcommand{\a}{{\bf a}}
\renewcommand{\b}{{\bf b}}
\renewcommand{\c}{{\bf c}}
\renewcommand{\d}{{\rm d}}  
\newcommand{\e}{{\bf e}}
\newcommand{\f}{{\bf f}}
\newcommand{\g}{{\bf g}}
\newcommand{\h}{{\bf h}}
\newcommand{\bi}{{\bf i}}
\newcommand{\bj}{{\bf j}} 

\renewcommand{\k}{{\bf k}}
\newcommand{\m}{{\bf m}}
\newcommand{\mhat}{{\overline{m}}}
\newcommand{\tm}{{\tilde{m}}}
\newcommand{\n}{{\bf n}}
\renewcommand{\o}{{\bf o}}
\newcommand{\p}{{\bf p}}
\newcommand{\q}{{\bf q}}
\newcommand{\wy}{{\widehat{\y}}}
\newcommand{\wlam}{{\widehat{\lambda}}}
\renewcommand{\r}{{\bf r}}
\newcommand{\s}{{\bf s}}
\renewcommand{\t}{{\bf t}}
\renewcommand{\u}{{\bf u}}
\renewcommand{\v}{{\bf v}}
\newcommand{\w}{{\bf w}}
\newcommand{\x}{{\bf x}}
\newcommand{\y}{{\bf y}}
\newcommand{\z}{{\bf z}}
\newcommand{\A}{{\bf A}}
\newcommand{\B}{{\bf B}}
\newcommand{\C}{{\bf C}}
\newcommand{\D}{{\bf D}}
\newcommand{\F}{{\bf F}}
\newcommand{\G}{{\bf G}}
\newcommand{\Gcal}{{\mathcal{G}}}
\newcommand{\Dcal}{\mathcal{D}}
\newcommand{\Qcal}{{\mathcal{Q}}}
\newcommand{\Pcal}{{\mathcal{P}}}
\newcommand{\Hcal}{{\mathcal{H}}}
\renewcommand{\H}{{\bf H}}
\newcommand{\I}{{\bf I}}
\newcommand{\J}{{\bf J}}
\newcommand{\K}{{\bf K}}
\renewcommand{\L}{{\bf L}}
\newcommand{\Lcal}{{\mathcal{L}}}
\newcommand{\M}{{\bf M}}
\newcommand{\Mcal}{{\mathcal{M}}}
\newcommand{\Acal}{{\mathcal{A}}}
\newcommand{\Ocal}{{\mathcal{O}}}
\newcommand{\Fcal}{{\mathcal{F}}}
\newcommand{\N}{\mathcal{N}}  
\newcommand{\bupeta}{\boldsymbol{\upeta}}
\renewcommand{\O}{{\bf O}}
\renewcommand{\P}{{\bf P}}
\newcommand{\Q}{{\bf Q}}
\newcommand{\R}{{\bf R}}
\renewcommand{\S}{{\bf S}}
\newcommand{\Scal}{{\mathcal{S}}}
\newcommand{\T}{{\bf T}}
\newcommand{\Tcal}{{\mathcal{T}}}
\newcommand{\U}{{\bf U}}
\newcommand{\Ucal}{{\mathcal{U}}}
\newcommand{\tUcal}{{\tilde{\Ucal}}}
\newcommand{\V}{{\bf V}}
\newcommand{\W}{{\bf W}}
\newcommand{\Wcal}{{\mathcal{W}}}
\newcommand{\Vcal}{{\mathcal{V}}}
\newcommand{\X}{{\bf X}}
\newcommand{\Xcal}{{\mathcal{X}}}
\newcommand{\Y}{{\bf Y}}
\newcommand{\Ycal}{{\mathcal{Y}}}
\newcommand{\Z}{{\bf Z}}
\newcommand{\Zcal}{{\mathcal{Z}}}

\newcommand{\bfLambda}{\boldsymbol{\Lambda}}

\newcommand{\bsigma}{\boldsymbol{\sigma}}
\newcommand{\balpha}{\boldsymbol{\alpha}}
\newcommand{\bpsi}{\boldsymbol{\psi}}
\newcommand{\bphi}{\boldsymbol{\phi}}
\newcommand{\bPhi}{\boldsymbol{\Phi}}
\newcommand{\bbeta}{\boldsymbol{\beta}}
\newcommand{\Beta}{\boldsymbol{\eta}}
\newcommand{\btau}{\boldsymbol{\tau}}
\newcommand{\bvarphi}{\boldsymbol{\varphi}}
\newcommand{\bzeta}{\boldsymbol{\zeta}}

\newcommand{\blambda}{\boldsymbol{\lambda}}
\newcommand{\bLambda}{\mathbf{\Lambda}}

\newcommand{\btheta}{\boldsymbol{\theta}}
\newcommand{\bpi}{\boldsymbol{\pi}}
\newcommand{\bxi}{\boldsymbol{\xi}}
\newcommand{\bSigma}{\boldsymbol{\Sigma}}
\newcommand{\bPi}{\boldsymbol{\Pi}}
\newcommand{\bOmega}{\boldsymbol{\Omega}}

\newcommand{\bx}{{\bf x}}
\newcommand{\bgamma}{\boldsymbol{\gamma}}
\newcommand{\bGamma}{\boldsymbol{\Gamma}}
\newcommand{\bUpsilon}{\boldsymbol{\Upsilon}}

\newcommand{\bmu}{\boldsymbol{\mu}}
\newcommand{\0}{{\bf 0}}

\newcommand{\bs}{\backslash}
\newcommand{\ben}{\begin{enumerate}}
\newcommand{\een}{\end{enumerate}}

 \newcommand{\notS}{{\backslash S}}
 \newcommand{\nots}{{\backslash s}}
 \newcommand{\noti}{{\backslash i}}
 \newcommand{\notj}{{\backslash j}}
 \newcommand{\nott}{\backslash t}
 \newcommand{\notone}{{\backslash 1}}
 \newcommand{\nottp}{\backslash t+1}

\newcommand{\notk}{{^{\backslash k}}}
\newcommand{\notij}{{^{\backslash i,j}}}
\newcommand{\notg}{{^{\backslash g}}}
\newcommand{\wnoti}{{_{\w}^{\backslash i}}}
\newcommand{\wnotg}{{_{\w}^{\backslash g}}}
\newcommand{\vnotij}{{_{\v}^{\backslash i,j}}}
\newcommand{\vnotg}{{_{\v}^{\backslash g}}}
\newcommand{\half}{\frac{1}{2}}
\newcommand{\msgb}{m_{t \leftarrow t+1}}
\newcommand{\msgf}{m_{t \rightarrow t+1}}
\newcommand{\msgfp}{m_{t-1 \rightarrow t}}

\newcommand{\proj}[1]{{\rm proj}\negmedspace\left[#1\right]}

\newcommand{\dif}{\mathrm{d}}
\newcommand{\abs}[1]{\lvert#1\rvert}
\newcommand{\norm}[1]{\lVert#1\rVert}

\newcommand{\mrm}[1]{\mathrm{{#1}}}
\newcommand{\RomanCap}[1]{\MakeUppercase{\romannumeral #1}}
\newcommand{\EE}{\mathbb{E}}
\newcommand{\bbI}{\mathbb{I}}
\newcommand{\bbH}{\mathbb{H}}
\newcommand{\ie}{{\textit{i.e.,}}\xspace}
\newcommand{\eg}{{\textit{e.g.,}}\xspace}
\newcommand{\etc}{{\textit{etc.}}\xspace}
\newcommand{\cmt}[1]{}	
\begin{abstract}
Motivational-interviewing-based health coaching is an effective approach for improving mental health and promoting healthy behavior change. However, the scarcity of trained human coaches and the high cost of coaching services make such support inaccessible to many people who could benefit from it.
This motivates the development of AI health coaches that can provide scalable and affordable support. 
Existing methods typically optimize only one side of the interaction: they either train a dialogue agent against a fixed client environment or train a client simulator against a fixed assistant. This one-sided setup can limit exploration of the interaction space and may be inefficient at developing the capabilities required by the target agent and pushing its performance boundaries.  In this paper, we propose a dual-agent framework that interactively co-trains both the health coach agent and the client simulator. The coach is optimized with DPO using Pareto-dominant response pairs identified by a multi-dimensional LLM judge. In turn, the client is trained adversarially by reversing these preferences, inducing an implicit adversarial training dynamic. We further show that this co-training process admits a natural stochastic-game interpretation. 
Extensive experiments demonstrate that our method effectively improves coaching quality across several important dimensions.
\end{abstract}

\section{Introduction}
Motivational interviewing is a widely used health-coaching approach for improving mental health and promoting healthy behavior change~\citep{rollnick2022mihealthcare,miller2023motivational}. Its effectiveness depends not only on whether a conversation reaches a desirable final outcome, but also on whether the coach behaves appropriately throughout the interaction. At each turn, the coach must make context-sensitive decisions, such as eliciting and strengthening change talk, responding skillfully to sustain talk, expressing empathy, helping the client clarify needs and ambivalence, and, when appropriate, guiding the client toward concrete goals and action plans. In practice, however, the scarcity of trained human coaches and the high cost of coaching services make such support inaccessible to many people who could benefit from it. This creates an urgent need for AI health coaches that can provide scalable and affordable support.

Although many AI assistants have been developed for dialogue with human users~\citep{ouyang2022training,bai2022constitutional,rafailov2023dpo}, most existing methods focus primarily on optimizing conversation-level outcomes. For health coaching, however, success depends on the quality of the coach’s behavior throughout the entire interviewing session, not merely on the final result. Moreover, existing training methods typically optimize only one side of the interaction: they either train a dialogue agent against a fixed client environment~\citep{gao2025refuel,zhou2025sweet,zhao2025cso}, or train a client simulator against a fixed assistant~\citep{Zhu2026-or}. This asymmetric setup can limit exploration of the interaction space during training and may be inefficient in developing the capabilities required by the target agents and driving them toward their full potential.

In this paper, we propose \ours, a dual-agent co-training framework for motivational-interviewing-based health coaching. Our method jointly trains a client simulator and a coach agent through interactive dialogue, enabling richer exploration of interviewing dynamics and, in particular, of the challenging interaction patterns that are most informative for improving coaching quality. Our main contributions are as follows.
\begin{compactitem}
\item \textbf{Dual-Agent Preference Tree Construction.} Instead of relying on expensive expert demonstrations, we introduce an efficient and low-cost tree-based rollout procedure for collecting preference pairs during training. Each dialogue tree is initialized by prompting the client simulator with a short random persona generated by a low-cost third-party LLM, after which the coach and client interact to simulate a health-coaching interview. At selected dialogue states, the tree branches by re-sampling multiple coach utterances under the same interaction history and recursively expanding the resulting continuations. Preference pairs are then constructed from utterances at the same tree level according to their estimated Q-values.

\item \textbf{Multi-Dimensional LLM Judge for Q-Value Estimation.} We design a multi-dimensional LLM judge to evaluate the coach’s behavior throughout the dialogue. To ensure that each coaching dimension is evaluated in an appropriate context, the judging process is decomposed into three stages: client-state identification, state-conditioned sentence-level functional labeling, and final reasoning-based aggregation and scoring. The resulting node-level scores are propagated upward with a discount factor and aggregated into Q-values that represent the overall quality of candidate coach utterances.

\item \textbf{Implicit Adversarial Preference Training.} For the coach, we construct preference pairs by requiring the preferred response to outperform the rejected response on every judged dimension, yielding Pareto-dominant preference pairs for DPO training. This encourages the coach to improve across all targeted dimensions simultaneously. For the client simulator, we reverse the coach-side preferences and train it adversarially: responses that make the coach perform worse are treated as preferred. As a result, the client learns to produce utterances that are increasingly challenging for the current coach to handle appropriately, which in turn drives the coach toward stronger performance through co-training. We further show that this simple implicit adversarial dynamic admits a natural stochastic-game interpretation.

\item \textbf{Experiments.} We evaluate \ours on three core coaching dimensions: cultivating change talk, softening or skillfully responding to sustain talk, and empathy. Blind ranking studies show strong agreement between the preferences induced by our LLM judge and those of human judges. Across several challenging client conditions, \ours consistently improves coaching quality throughout the interview compared with carefully prompted state-of-the-art LLM baselines and a purely supervised fine-tuning (SFT) baseline. Extensive ablations further demonstrate that our co-training framework substantially outperforms both one-sided training and purely supervised fine-tuning.

\end{compactitem}

\section{Methodology}
\label{sec:method}

\begin{figure}[h!]
  \centering
\includegraphics[width=\linewidth]{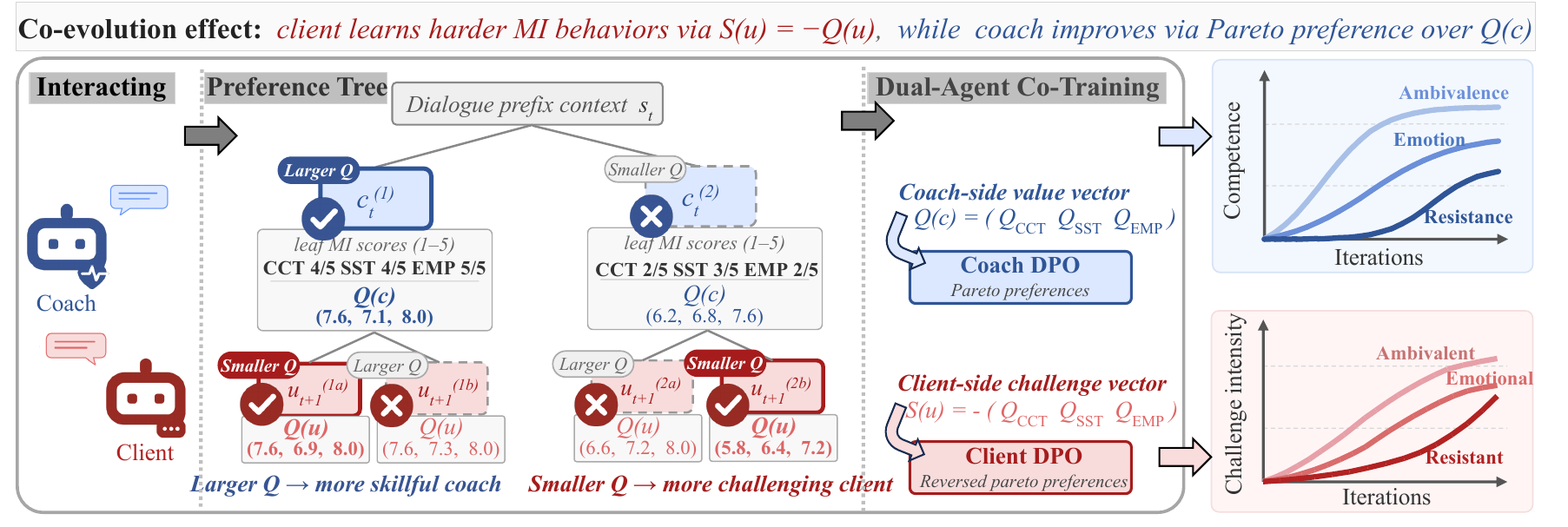}
  \caption{\small 
    \textbf{Overview of DACT.}
    \emph{Left: }at each co-evolution round, the coach adapter $\pi_C$ and client adapter $\pi_U$ engage in multi-turn dialogue.
    \emph{Middle:} the coach branches into candidate utterances $\{c_t^{(i)}\}$; each branch is followed by candidate client responses $\{u_{t+1}^{(ij)}\}$. A frozen LLM judge scores every coach node on a 1--5 scale, and the resulting scores are backed up into per-node Q-vectors $Q(c)$ and $Q(u)$.
    \emph{Right:} the coach is trained on \emph{Pareto-dominant} pairs over $Q(c)$ (chosen = larger $Q$, more skillful), while the client is trained on the \emph{same} tree under the sign-flipped signal $S(u) = -Q(u)$ (chosen = smaller $Q$, more challenging for the coach).
    \emph{Far right:} across iterations, the coach's competence and the client's challenge intensity grow in tandem along the resistance, emotion, and ambivalence dimensions of MI client states, empirically realizing the adversarial co-evolution that is central to DACT.
  }
  \label{fig:method}
\end{figure}

\textbf{Problem formulation.}
We formulate the co-training of the client simulator and the coach agent as a two-agent reinforcement learning problem, $\mathcal{G} = (\mathcal{S}, \mathcal{A}^C, \mathcal{A}^U, P, \mathbf{R}, \gamma, T)$,
where $\mathcal{S}$ denotes the state space, $\mathcal{A}^C$ and $\mathcal{A}^U$ denote the action spaces of the coach and the client, respectively, $P$ denotes the transition dynamics, $\mathbf{R}$ denotes the reward signal, $\gamma \in (0,1]$ is the discount factor, and $T$ is the maximum horizon, i.e., the maximum number of coach turns in an interview session.
Each step corresponds to one coach turn in the dialogue. At step $t$, the state $s_t \in \mathcal{S}$ is the dialogue history immediately before the coach acts. Given $s_t$, the coach generates an utterance $c_t \in \mathcal{A}^C$ according to its policy $\pi_C(\cdot \mid s_t)$. The client then responds with an utterance $u_{t+1} \in \mathcal{A}^U$ according to its policy $\pi_U(\cdot \mid s_t, c_t)$. The dialogue history is updated as $s_{t+1} = (s_t, c_t, u_{t+1})$. 
The interaction terminates when the coach emits a termination token or when the maximum number of turns $T$ is reached.

Unlike prior approaches that optimize only one side of the interaction while fixing the other, our goal is to jointly learn both the coach policy $\pi_C$ and the client policy $\pi_U$. Consequently, from the perspective of the coach, the effective environment is non-stationary, since the transition dynamics are induced by the evolving client policy: $P(s_{t+1} \mid s_t, c_t) = \pi_U(u_{t+1} \mid s_t, c_t)$.
This co-evolving interaction structure enriches exploration of the dialogue space and exposes the coach to increasingly challenging conversational situations. As we show later, our joint training process admits a natural stochastic-game interpretation.

A well-executed motivational interview requires the coach to perform well across multiple dimensions throughout the entire dialogue, rather than merely achieving a desirable final outcome. We therefore introduce a vector-valued reward signal to evaluate the quality of each coach action: $\mathbf{R} = (r_1, \ldots, r_d)$.
In this work, we focus on three key dimensions ($d=3$), selected based on the Motivational Interviewing Treatment Integrity (MITI) 4.2.1 coding manual~\citep{moyers2014motivational}: \emph{Cultivating Change Talk} (CCT), \emph{Softening Sustain Talk} (SST), and \emph{Empathy}. Our framework can be naturally extended to incorporate additional dimensions. Each reward component is scored on a five-point scale, $r_i \in \{1,2,3,4,5\}$,
and the semantic meaning of each score is provided in Appendix~\ref{sect:appendix:reward-scale}.

\subsection{Preference Dialogue Tree Construction}

After initialization---typically through lightweight supervised fine-tuning of a small LLM---we iteratively refine both policies using direct preference optimization (DPO)~\citep{rafailov2023dpo}. At each iteration, preference pairs are collected from a batch of dialogue trees generated by the previous client and coach policies.

Specifically, at iteration $k$, we construct each dialogue tree by letting the coach and client agents initiate an interview under their previous policies $(\pi_C^{k-1}, \pi_U^{k-1})$. The dialogue first unfolds linearly until it reaches a \emph{branching point}, namely, a selected coach turn at which branching begins. At that point, the coach samples multiple candidate utterances conditioned on the same dialogue history, each of which defines a separate branch. For each such coach branch, the client then independently samples multiple responses, yielding further branches in the rollout tree. This branching-and-rollout process continues recursively until each dialogue path terminates or a predefined maximum tree depth is reached.  To ensure broad coverage of conversational dynamics, the branching point is selected from a range of dialogue steps, so that both early-stage and late-stage coaching decisions can be explored. The resulting balanced dialogue tree ensures that sibling nodes---utterances generated at the same depth under the same dialogue prefix---are compared under equal rollout budgets. This allows us to construct preference pairs in a controlled and less biased manner.

\subsection{Multi-Dimensional LLM Judge}
\label{sec:main_llm_judge}
While human experts can provide high-quality and clinically informed evaluations of coaching performance, such evaluation is expensive and time-consuming. To enable scalable learning, we develop a multi-dimensional LLM judge that evaluates coach behavior throughout the dialogue tree. These turn-level scores are then propagated through the tree and aggregated into Q-value estimates. To support clinically meaningful evaluation, we encode extensive motivational interviewing (MI) knowledge~\citep{moyers2014motivational}---including concepts such as change talk \textit{vs.} sustain talk, the righting reflex, and OARS micro-skills---into the judge prompt, and perform a three-step evaluation procedure.

\paragraph{Step 1: Client-state identification.}
The judge first classifies the client's most recent utterance into one of five states: \emph{engaged} (clear change talk), \emph{ambivalent} (both change talk and sustain talk are present), \emph{sustain-resistant} (reasons against change or explicit pushback), \emph{emotional} (difficult emotions are in the foreground), or \emph{informational} (neutral background sharing). This step is important because an MI-appropriate coach response is highly state-dependent: what is skillful for an engaged client (\eg reinforcing change talk) may be premature for an emotional client, where the coach should first acknowledge and sit with the feeling rather than immediately evoke change.

\paragraph{Step 2: State-conditioned sentence-level function labeling.}
Conditioned on the identified client state, the judge then decomposes the coach utterance into individual sentences and assigns each sentence a function label from a 17-category taxonomy organized into three groups:

\begin{itemize}[leftmargin=1.5em, topsep=2pt, itemsep=1pt]
    \item \textbf{Group A} (MI-adherent, 6 categories): \texttt{open\_question\_evoking}, \texttt{simple\_reflection}, \texttt{complex\_reflection}, \texttt{double\_sided\_reflection}, \texttt{affirmation}, \texttt{summary}.
    \item \textbf{Group B} (neutral, 3 categories): \texttt{neutral\_question}, \texttt{rapport\_or\_info}, \texttt{other}.
    \item \textbf{Group C} (MI-non-adherent, 8 categories): \texttt{leading\_question}, \texttt{premature\_planning}, \texttt{arguing\_for\_change}, \texttt{arguing\_against\_sustain}, \texttt{reassurance}, \texttt{unsolicited\_advice}, \texttt{parroting}, \texttt{distorted\_reflection}.
\end{itemize}

\paragraph{Step 3: Reasoning-based aggregation and scoring.}
Finally, for each target dimension---CCT, SST, and Empathy---the judge produces detailed reasoning that references the relevant sentences and their function labels, identifies the applicable scoring rules, and aggregates them into a reward score in $\{1,2,3,4,5\}$. The output schema enforces that the reasoning fields appear before the score fields, which discourages post-hoc rationalization and makes the scoring process more transparent.

\noindent\textbf{Q-value estimation.}
Let $\mathbf{r}(n) = (r_1(n), r_2(n), r_3(n))$ denote the immediate reward vector at node $n$. A node may correspond to either a coach utterance or a client utterance. Since our objective is to evaluate coaching quality, only coach nodes receive nonzero immediate rewards from the LLM judge; for client nodes, we set $\mathbf{r}(n) = \mathbf{0}$.

The immediate judge scores capture turn-level coaching quality, but they do not reflect the downstream effect of a coach utterance on future interaction trajectories. For example, an open question may receive only a moderate immediate score, yet elicit change talk from the client and thereby enable stronger coach responses in later turns. To account for such delayed effects, we propagate turn-level scores upward through the dialogue tree using a discounted recursive backup. Specifically, for each node $n$ and each dimension $d$, we compute
\begin{equation}
    Q_d(n) = 
    \begin{cases}
        r_d(n), & \text{if $n$ is a leaf}, \\[4pt]
        r_d(n) + \gamma \cdot \displaystyle\frac{1}{|\mathcal{C}(n)|} \sum_{n' \in \mathcal{C}(n)} Q_d(n'), & \text{otherwise},
    \end{cases}
    \label{eq:bellman}
\end{equation}
where $\mathcal{C}(n)$ denotes the set of children of node $n$, and $\gamma$ is the discount factor. This yields a dimension-wise Q-value that reflects both the immediate quality of the current response and its expected downstream effect on future dialogue development.

\subsection{Implicit Adversarial Preference Training}\label{sect:adv-train}

\textbf{Coach preference pairs.}
To train the coach with DPO, we compare coach nodes at the same tree level within each dialogue tree and construct \emph{Pareto-dominant preference pairs} using strict Pareto dominance over their Q-value vectors. Specifically, for two coach nodes $A$ and $B$, we say that $A$ dominates $B$ (denoted by $A \succ B$) if
\begin{equation}
    \forall d \in D:\ Q_d(A) \geq Q_d(B)
    \quad \wedge \quad
    \exists d \in D:\ Q_d(A) > Q_d(B).
    \label{eq:pareto}
\end{equation}
When $A \succ B$, we add the pair $(A,B)$ to the coach preference dataset, with $A$ as the chosen response and $B$ as the rejected response. This Pareto-based construction encourages the coach to improve coaching quality across \emph{all} target dimensions, rather than over-optimizing a single dimension or improving one dimension at the expense of another.

\textbf{Client preference pairs.}
We define a challenge vector for each client node $u$ in the dialogue tree:
\[
S(u) = \bigl(-Q_1(u), -Q_2(u), -Q_3(u)\bigr).
\]
Since $\{Q_d(u)\}$ reflects the downstream coaching quality following the client utterance at node $u$, a smaller value of $Q_d(u)$---equivalently, a larger value of the corresponding component in $S(u)$---indicates that the subsequent coach behavior is worse on dimension $d$. Therefore, a larger challenge vector corresponds to a client utterance that is more challenging for the current coach.

We compare client nodes among comparable nodes at the same tree depth and construct \emph{Pareto-dominant preference pairs} using the challenge vectors. Specifically, for two client nodes $u_1$ and $u_2$, we prefer $u_1$ over $u_2$ if
\begin{equation}
\forall d \in D:\ S_d(u_1) \geq S_d(u_2)
\quad \wedge \quad
\exists d \in D:\ S_d(u_1) > S_d(u_2).
\label{eq:challenges}
\end{equation}
Such pairs are added to the client preference dataset with $u_1$ as the chosen response and $u_2$ as the rejected response. In other words, client-side preferences are constructed by reversing the coach-side notion of quality: utterances are preferred when they make the downstream coaching outcome worse across all target dimensions.

Both the coach and the client are then optimized using the standard DPO objective:
\begin{equation}
    \mathcal{L}_{\text{DPO}}(\pi_\theta; \pi_{\text{ref}})
    =
    -\mathbb{E}_{(x,y_w,y_l)}
    \left[
    \log \sigma \left(
    \beta
    \left(
    \log \frac{\pi_\theta(y_w \mid x)}{\pi_{\text{ref}}(y_w \mid x)}
    -
    \log \frac{\pi_\theta(y_l \mid x)}{\pi_{\text{ref}}(y_l \mid x)}
    \right)
    \right)
    \right],
    \label{eq:dpo}
\end{equation}
where $\pi_{\text{ref}} = \pi_\theta^{k-1}$ (or $\pi^{\text{SFT}}_U$ for the client; see Appendix~\ref{sec:appendix:hyperparams}) is the policy from the previous training round, $y_w$ and $y_l$ denote the chosen and rejected utterances, respectively, and $x$ is the dialogue history. Our overall method is summarized in Appendix Algorithm~\ref{alg:main}.

This construction induces an implicit adversarial dynamic between the two agents. The coach is trained to prefer responses that improve multi-dimensional coaching quality, while the client is trained to prefer utterances that make the subsequent coaching interaction more difficult. In effect, the client learns to produce utterances that challenge the coach, such as sustain talk that resists reframing, ambivalence that resists easy resolution, or emotionally complex statements that require more sophisticated empathy. The adversarial pressure is \emph{implicit} because it arises through the Q-value-based preference construction rather than through an explicit adversarial objective. As the coach improves and resolves its current weaknesses, those failure modes no longer induce low Q-values; consequently, the client’s preference landscape automatically shifts toward exposing the next set of weaknesses. This produces a self-paced adversarial curriculum without requiring manual difficulty scheduling.

\subsection{Stochastic-Game Interpretation}
\label{sec:game-interpretation}

Although our method is implemented through preference optimization rather than direct game-theoretic optimization, the resulting co-training dynamics admit a natural stochastic-game interpretation. The key observation is that the coach-side and client-side preference pairs are consistent with two coupled latent utilities: the coach prefers utterances that improve downstream multi-dimensional coaching quality, whereas the client prefers utterances that reduce it.

To formalize this idea, let $\mathcal{D}=\{1,\ldots,d\}$ denote the set of reward dimensions, and let $\Phi:\mathbb{R}^d \to \mathbb{R}$ be any \emph{strictly Pareto-monotone scalarization}, meaning that for any $x,y\in\mathbb{R}^d$,
\[
\bigl(\forall i,\ x_i \ge y_i \bigr)\ \wedge\ \bigl(\exists j,\ x_j > y_j \bigr)
\quad \Longrightarrow \quad
\Phi(x) > \Phi(y).
\]
We emphasize that $\Phi$ is introduced only for interpretation and is \emph{not} used in training. Using $\Phi$, we define the latent coach and client utilities as
\begin{equation}
    U_C(n) = \Phi(Q(n)),
    \qquad
    U_U(n) = -\,\Phi(Q(n)),
    \label{eq:latent-utility-main}
\end{equation}
where $Q(n)=(Q_1(n),\ldots,Q_d(n))$ is the dimension-wise backed-up return at node $n$.

\begin{lemma}[Coach-side utility consistency]
\label{lem:coach-consistency-main}
If two comparable coach nodes $A$ and $B$ satisfy $A \succ B$ in the strict Pareto sense of Eq.~\ref{eq:pareto}, then $U_C(A) > U_C(B)$.
\end{lemma}

Lemma~\ref{lem:coach-consistency-main} shows that every Pareto-dominant coach preference pair is consistent with maximizing a scalar utility that is monotone in all coaching dimensions. Thus, the coach-side DPO update can be interpreted as increasing the probability of actions with higher latent coach utility.

\begin{lemma}[Client-side adversarial consistency]
\label{lem:client-consistency-main}
Let $u_1$ and $u_2$ be two comparable client nodes. If the corresponding challenge vectors satisfy $S(u_1) \succ S(u_2)$,
then $U_U(u_1) > U_U(u_2)$, or equivalently, $U_C(u_1) < U_C(u_2)$.
\end{lemma}

Lemma~\ref{lem:client-consistency-main} formalizes the adversarial role of the client simulator: client utterances are preferred precisely when they reduce the downstream latent utility of the coach.

Together, these two lemmas imply that our coach-side and client-side preference datasets are utility-consistent with a two-player stochastic game defined on the same dialogue state-transition structure as $\mathcal{G}$, but with opposing latent utilities $(U_C,U_U)$. Consequently, alternating DPO updates for the coach and the client can be interpreted as alternating KL-regularized policy-improvement steps for two adversarially coupled players. Detailed proofs and discussion are provided in Appendix~\ref{sec:appendix:game-interpretation}.

\section{Related Work}\label{sec:related}

Our work is most closely related to recent efforts on tree-based and multi-turn preference learning for LLMs. MCTS-DPO~\citep{xie2024mctsdpo} uses Monte Carlo tree search to construct step-level preference pairs for single-agent reasoning, showing that tree rollouts can provide finer-grained supervision than outcome-only rewards. REFUEL~\citep{gao2025refuel} and SWEET-RL~\citep{zhou2025sweet} extend preference- or value-based optimization to multi-turn interaction, addressing long-horizon credit assignment in dialogue and collaborative agent tasks. In therapeutic dialogue, CSO~\citep{zhao2025cso} applies MCTS and DPO to emotional support conversations by learning turn-level strategy preferences. Our method is related to these lines of work in using rollout-based credit assignment and preference optimization, but differs in three key respects: we study \emph{dual-agent} health-coaching dialogue rather than single-agent reasoning or fixed-user interaction; we construct \emph{strict Pareto-dominant} pairs from multi-dimensional MI rewards rather than scalar values; and we optimize utterance-level coaching behavior instead of strategy selection alone.

A second related thread concerns user simulation and interactive social training. PlatoLM~\citep{kong2024platolm} trains a dedicated user simulator to generate more human-like multi-round dialogue data, while SOTOPIA-$\pi$~\citep{wang2024sotopiapi} improves social agents through behavior cloning and self-reinforcement on LLM-rated interaction trajectories. DIAL~\citep{Zhu2026-or} iteratively trains a realistic dialogue simulator through adversarial preference learning, and Roleplay-doh~\citep{louie2024roleplaydoh} uses expert-authored prompting principles to build clinically plausible simulated patients. These works demonstrate the value of learned or expert-guided simulators, but they either train only the simulator, keep the assistant fixed, or rely on filtered positive trajectories rather than pairwise preference optimization. In contrast, our framework jointly updates both the coach and the client simulator, and the client is trained not only for realism but also to expose the current weaknesses of the coach, yielding an implicit adversarial curriculum during co-training.

Finally, our work connects to preference optimization and multi-objective alignment. DPO~\citep{rafailov2023dpo} provides the basic optimization algorithm we use for both agents. Prior multi-objective alignment methods such as MODPO~\citep{zhou2024modpo} typically rely on scalarization or weighted trade-offs across objectives. By contrast, our coach-side supervision is based on strict Pareto dominance across clinically grounded MI dimensions, avoiding the need to specify task-dependent weights between cultivating change talk, softening sustain talk, and empathy. More broadly, our use of an LLM judge also relates to the RLAIF and self-rewarding paradigm~\citep{bai2022constitutional,yuan2024selfrewarding}, but we apply it in a structured, sentence-aware, multi-dimensional way tailored to motivational interviewing rather than generic helpfulness or single-turn alignment.
\section{Experiments}
\label{sec:experiments}

\subsection{Setup}
\label{sec:setup}

\textbf{Personas and models.}
We construct a pool of structured client personas as role-play conditioning for the client side. These personas are partitioned into disjoint sets to generate the SFT corpus, to condition the client agent during preference tree generation, and to drive client simulators at evaluation.
Both the coach and client agents share \texttt{Qwen2.5-32B-Instruct-GPTQ-Int4}~\citep{qwen2025qwen25technicalreport} as the base, with two independent LoRA~\citep{hu2022lora} adapters $\pi_C, \pi_U$ attached to all attention and MLP projections, trained with a single NVIDIA H200 GPU. The two adapters are initialized by supervised fine-tuning on a synthetic MI dialogue corpus, produced by prompting two GPT-4o-mini agents to role-play a coach and a sampled persona.

\textbf{Co-evolution.} At each iteration, we generate $N{=}3$ preference trees. Each tree is constructed through four alternating branching steps between the coach and client agents. At every step, each agent expands $M{=}3$ child nodes, yielding $M^{4}{=}81$ leaf trajectories per tree. Q-values are then backed up using the discounted recursion in Eq.~\eqref{eq:bellman}, and both adapters are optimized with DPO, using GPT-5.4-mini as the judge. We run 13 co-training iterations in total. Full hyperparameters, prompt templates, data-generation details, and the scoring protocol are provided in Appendix~\ref{sec:appendix:pipeline}.

\subsection{Evaluation Protocol}
\label{sec:eval-protocol}

\textbf{Competing methods.} We compare \textsc{DACT} against three baselines. (1) \textsc{SFT}: the supervised fine-tuned coach $\pi_C^{\mathrm{SFT}}$ without DPO, isolating what co-evolution adds. (2) \textsc{GPT-Coach}: GPT-5.5 prompted with the same system prompt that conditions \textsc{DACT} at inference, providing a matched-prompt control: both methods receive identical instruction at inference time, so any quality gap reflects what training has internalized into $\pi_C$ rather than what the prompt itself elicits. (3) \textsc{GPT-Rubric}: GPT-5.5 conditioned on a longer prompt that embeds the LLM-judge scoring rubric verbatim, giving direct access to the evaluation criteria; this serves as an upper reference for what prompt engineering alone can achieve given the strongest possible specification of the target behavior.

\textbf{Client conditions.} Each coach is evaluated against four clients spanning in-distribution stress and OOD generalization. \textsc{$R_8$-Client} is the co-evolved client from training round 8, identified as the most challenging client across the trajectory by an SFT-fixed-coach probe (Appendix~\ref{sec:appendix:probe}); using $R_8$ rather than the final 
client tests the trained coach against a peak-difficulty in-distribution opponent. The remaining three are out-of-distribution simulators implemented by applying GPT-5.4-mini under three distinct prompts: \textsc{OAI-Emotional} foregrounds difficult emotion that demands acknowledgment before evocation, \textsc{OAI-Resistant} emphasizes precise factual constraints and refuses generic advice, and \textsc{OAI-Ambivalent} holds change talk and sustain talk in equal weight. The four conditions test whether \textsc{DACT}'s coaching quality survives both the hardest client it was trained against and three principled OOD client styles it never saw during training.

\textbf{Metrics.} Each evaluation dialogue runs $T{=}8$ coach steps at temperature $0.2$, with utterances at $t \in [4, 8]$ scored by the LLM judge (Section~\ref{sec:main_llm_judge}). We report two complementary metrics per cell. \emph{mean3} is the cell-level mean of the per-utterance average of CCT, SST, and Empathy on the $1$--$5$ scale, with each dimension also reported separately. \emph{anti\%} is the sentence-level rate of MI anti-patterns---the fraction of sentences whose Step-$2$ function label falls in the eight-class Group~C taxonomy. Full evaluation details are provided in Appendix~\ref{sec:appendix:eval}.
\subsection{Main Comparison}
\label{sec:main-comparison}

Table~\ref{tab:main-results} reports per-dimension scores, \textit{mean3}, and \textit{anti}\% for all the methods, with aggregate summaries across all client conditions.

\begin{table}[t]
\centering
\caption{\small Main comparison: four coach methods on four client conditions, $100$ scored utterances per cell ($20$ personas $\times$ $5$ steps, $t \in [4,8]$). Best values per (client, metric) cell are bolded; the bottom block averages \textit{mean3} and \textit{anti}\% across all four conditions. Cells report mean $\pm$ standard error across personas.}
\vspace{0.1in}
\small
\renewcommand{\arraystretch}{1.00}
\setlength{\tabcolsep}{2.0pt}

\begin{tabular}{@{}llccccc@{}}
\toprule
\textbf{Client} & \textbf{Method} & \textbf{CCT} & \textbf{SST} & \textbf{Empathy} & \textbf{mean3} & \textbf{anti\%} \\
\midrule
\multirow{4}{*}{\textsc{$R_8$-Client}}
 & \textsc{DACT}       & $\mathbf{3.88}${\scriptsize $\pm0.06$} & $\mathbf{4.40}${\scriptsize $\pm0.06$} & $4.46${\scriptsize $\pm0.05$} & $\mathbf{4.25}${\scriptsize $\pm0.05$} & $\mathbf{0.41}${\scriptsize $\pm0.25$} \\
 & \textsc{GPT-Rubric} & $2.52${\scriptsize $\pm0.11$} & $3.01${\scriptsize $\pm0.10$} & $\mathbf{4.47}${\scriptsize $\pm0.05$} & $3.33${\scriptsize $\pm0.08$} & $16.15${\scriptsize $\pm3.65$} \\
 & \textsc{GPT-Coach}  & $2.68${\scriptsize $\pm0.12$} & $2.97${\scriptsize $\pm0.12$} & $4.19${\scriptsize $\pm0.03$} & $3.28${\scriptsize $\pm0.08$} & $9.57${\scriptsize $\pm1.79$} \\
 & \textsc{SFT}        & $1.83${\scriptsize $\pm0.04$} & $2.13${\scriptsize $\pm0.04$} & $3.76${\scriptsize $\pm0.05$} & $2.57${\scriptsize $\pm0.03$} & $39.00${\scriptsize $\pm2.27$} \\
\midrule
\multirow{4}{*}{\textsc{OAI-Emotional}}
 & \textsc{DACT}       & $\mathbf{3.80}${\scriptsize $\pm0.06$} & $\mathbf{4.39}${\scriptsize $\pm0.05$} & $4.43${\scriptsize $\pm0.05$} & $\mathbf{4.21}${\scriptsize $\pm0.04$} & $\mathbf{0.26}${\scriptsize $\pm0.26$} \\
 & \textsc{GPT-Rubric} & $3.49${\scriptsize $\pm0.12$} & $4.32${\scriptsize $\pm0.13$} & $\mathbf{4.62}${\scriptsize $\pm0.09$} & $4.14${\scriptsize $\pm0.11$} & $6.52${\scriptsize $\pm1.66$} \\
 & \textsc{GPT-Coach}  & $2.99${\scriptsize $\pm0.14$} & $3.73${\scriptsize $\pm0.13$} & $4.40${\scriptsize $\pm0.06$} & $3.71${\scriptsize $\pm0.10$} & $7.80${\scriptsize $\pm2.06$} \\
 & \textsc{SFT}        & $2.03${\scriptsize $\pm0.05$} & $2.35${\scriptsize $\pm0.07$} & $3.49${\scriptsize $\pm0.05$} & $2.62${\scriptsize $\pm0.04$} & $13.84${\scriptsize $\pm1.90$} \\
\midrule
\multirow{4}{*}{\textsc{OAI-Resistant}}
 & \textsc{DACT}       & $\mathbf{3.74}${\scriptsize $\pm0.11$} & $\mathbf{3.86}${\scriptsize $\pm0.09$} & $4.04${\scriptsize $\pm0.07$} & $\mathbf{3.88}${\scriptsize $\pm0.08$} & $\mathbf{0.66}${\scriptsize $\pm0.36$} \\
 & \textsc{GPT-Rubric} & $3.70${\scriptsize $\pm0.12$} & $3.47${\scriptsize $\pm0.09$} & $\mathbf{4.20}${\scriptsize $\pm0.04$} & $3.79${\scriptsize $\pm0.07$} & $25.51${\scriptsize $\pm4.35$} \\
 & \textsc{GPT-Coach}  & $3.39${\scriptsize $\pm0.17$} & $3.17${\scriptsize $\pm0.10$} & $4.09${\scriptsize $\pm0.03$} & $3.55${\scriptsize $\pm0.10$} & $28.27${\scriptsize $\pm3.73$} \\
 & \textsc{SFT}        & $2.54${\scriptsize $\pm0.11$} & $2.51${\scriptsize $\pm0.09$} & $3.40${\scriptsize $\pm0.05$} & $2.82${\scriptsize $\pm0.08$} & $36.24${\scriptsize $\pm3.44$} \\
\midrule
\multirow{4}{*}{\textsc{OAI-Ambivalent}}
 & \textsc{DACT}       & $\mathbf{4.49}${\scriptsize $\pm0.06$} & $4.70${\scriptsize $\pm0.04$} & $4.70${\scriptsize $\pm0.05$} & $4.63${\scriptsize $\pm0.04$} & $\mathbf{0.00}${\scriptsize $\pm0.00$} \\
 & \textsc{GPT-Rubric} & $4.40${\scriptsize $\pm0.04$} & $\mathbf{4.76}${\scriptsize $\pm0.06$} & $\mathbf{4.92}${\scriptsize $\pm0.03$} & $\mathbf{4.69}${\scriptsize $\pm0.03$} & $1.97${\scriptsize $\pm0.64$} \\
 & \textsc{GPT-Coach}  & $3.91${\scriptsize $\pm0.11$} & $4.39${\scriptsize $\pm0.08$} & $4.87${\scriptsize $\pm0.03$} & $4.39${\scriptsize $\pm0.07$} & $0.84${\scriptsize $\pm0.46$} \\
 & \textsc{SFT}        & $2.60${\scriptsize $\pm0.07$} & $2.61${\scriptsize $\pm0.08$} & $3.91${\scriptsize $\pm0.05$} & $3.04${\scriptsize $\pm0.05$} & $12.73${\scriptsize $\pm1.40$} \\
\bottomrule
\end{tabular}

\vspace{2pt} 

\setlength{\tabcolsep}{6pt} 
\begin{tabular}{@{}lcccc@{}}
\toprule
\textit{4-cond avg} & \textsc{DACT} & \textsc{GPT-Rubric} & \textsc{GPT-Coach} & \textsc{SFT} \\
\midrule
\textit{mean3}   & $\mathbf{4.24}${\scriptsize $\pm0.03$} & $3.99${\scriptsize $\pm0.04$}  & $3.73${\scriptsize $\pm0.04$}  & $2.77${\scriptsize $\pm0.03$} \\
\textit{anti\%}  & $\mathbf{0.33}${\scriptsize $\pm0.13$} & $12.54${\scriptsize $\pm1.49$} & $11.62${\scriptsize $\pm1.16$} & $25.45${\scriptsize $\pm1.19$} \\
\bottomrule
\end{tabular}
\label{tab:main-results}
\end{table}

\textbf{Overall coaching quality.} 
\textsc{DACT} achieves the highest \textit{mean3} under three of the four client conditions, exceeding the second-best method by margins of $0.92$, $0.07$, and $0.09$ on \textsc{$R_8$-Client}, \textsc{OAI-Emotional}, and \textsc{OAI-Resistant}, respectively. The most pronounced gain appears on \textsc{$R_8$-Client}, where \textsc{DACT} reaches a \textit{mean3} of $4.25$, versus $3.33$ and $3.28$ for the two GPT baselines---nearly a full-point improvement on the $1$--$5$ scale. This gap persists even when the GPT baselines are enhanced with inference-time prompting that inlines the scoring rubric verbatim.
One possible explanation is that \textsc{$R_8$-Client} was intentionally trained, via our co-evolution strategy, to elicit low-quality coach responses. Consequently, the frozen GPT-5.5 struggles with this type of difficult client behavior even when the rubric is explicitly inlined, whereas \textsc{DACT} has acquired the ability to handle such behavior through co-training with such clients.

The per-dimension breakdown localizes the source of the gap. \textsc{DACT}'s advantage is concentrated in CCT and SST: on \textsc{$R_8$-Client}, it scores $3.88$ and $4.40$, versus \textsc{GPT-Rubric}'s $2.52$ and $3.01$, despite the rubric being inlined. Empathy, by contrast, is nearly matched across the strongest three methods, staying above $4.0$. This is consistent with the roles of the three dimensions: empathy mainly reflects listening accuracy that a strong frozen LLM can already provide, whereas CCT and SST capture timing-sensitive MI skills---appropriately eliciting change talk and maintaining emotional safety---that benefit more from preference training. Aggregated over all four conditions, \textsc{DACT} reaches a \textit{mean3} of $4.24$, outperforming the next-best baseline (\textsc{GPT-Rubric}, $3.99$) by $0.25$.

\textbf{Anti-pattern suppression.}
\textsc{DACT} produces MI anti-patterns at sentence-level rates of $0.41\%$, $0.26\%$, $0.66\%$, and $0.00\%$ across the four conditions, averaging only $0.33\%$. In contrast, the two GPT baselines are one to two orders of magnitude worse (\textsc{GPT-Rubric}: $12.54\%$, \textsc{GPT-Coach}: $11.62\%$ on average), and both exceed $25\%$ on \textsc{OAI-Resistant}. \textsc{SFT} performs even worse, with an average anti-pattern rate of $25.45\%$, indicating that SFT alone leaves substantial MI-inconsistent behavior that DPO co-evolution largely removes. Notably, on \textsc{OAI-Resistant}, \textsc{GPT-Rubric} attains a relatively high \textit{mean3} of $3.79$, only $0.09$ below \textsc{DACT}'s $3.88$, yet its anti-pattern rate reaches $25.51\%$, more than $38\times$ higher than \textsc{DACT}'s $0.66\%$. 

Thus, \textsc{DACT} improves coaching quality while simultaneously suppressing anti-patterns, achieving a $20$--$40\times$ reduction in anti-pattern rate relative to the next-best method across all conditions, which indicates robust performance across diverse client behaviors.

\noindent\textbf{Alignment with human expert.}
To verify that the coach is learning preferences consistent with expert MI judgment, we sample preference pairs from across the training trajectory using our LLM judge. These pairs are blinded and independently ranked by a trained MI practitioner. As shown in Appendix Table~\ref{tab:human-llm-agreement-perround}, the expert’s rankings agree strongly with those of our LLM judge, and this agreement remains relatively consistent across different levels of score discrepancy. 

Additional details and discussion are provided in Appendix~\ref{sec:appendix:judge-validation}.

\subsection{Co-Evolution Dynamics and Ablation}
\label{sec:co-evolution-dynamics}

We next examine how our co-evolution strategy progressively improves the coach agent. To isolate this effect, we evaluate the checkpoint $\pi_C^{k}$ from each iteration $k$ against a single fixed client: the iteration-$8$ co-evolved client adapter, denoted by $\pi_U^{8}$. We select this client using an SFT-fixed-coach probe (Appendix~\ref{sec:appendix:probe}): holding the SFT coach $\pi_C^{\mathrm{SFT}}$ fixed and evaluating it against each iteration's client $\pi_U^{k}$, we find that the iteration-$8$ client induces the highest fraction of coach utterances with CCT scores below $3$ ($89.5\%$) and the highest proportion of emotion-dominated client states ($82.5\%$ emotional turns). We therefore use $\pi_U^{8}$ as a challenging in-distribution test client for analyzing how the coach improves over co-evolution.

Figure~\ref{fig:trajectory} plots the \textit{mean3} and \textit{anti}\% trajectories of three methods: \textbf{\textsc{DACT}}, our full co-evolution approach in which both adapters are trained; \textbf{Client-Frozen}, an ablation that disables client co-evolution by freezing the client after SFT while continuing to train the coach with DPO; and \textbf{SFT}, which applies no DPO and is shown as a fixed horizontal reference. Full per-iteration trajectories, per-dimension breakdowns (Appendix Figure~\ref{fig:per-dim-trajectory}, Table~\ref{tab:per-dim-trajectory}), and the SFT-fixed-coach probe are provided in Appendix~\ref{sec:appendix:trajectory}.

\begin{figure}[t]
\centering
\includegraphics[width=0.49\linewidth]{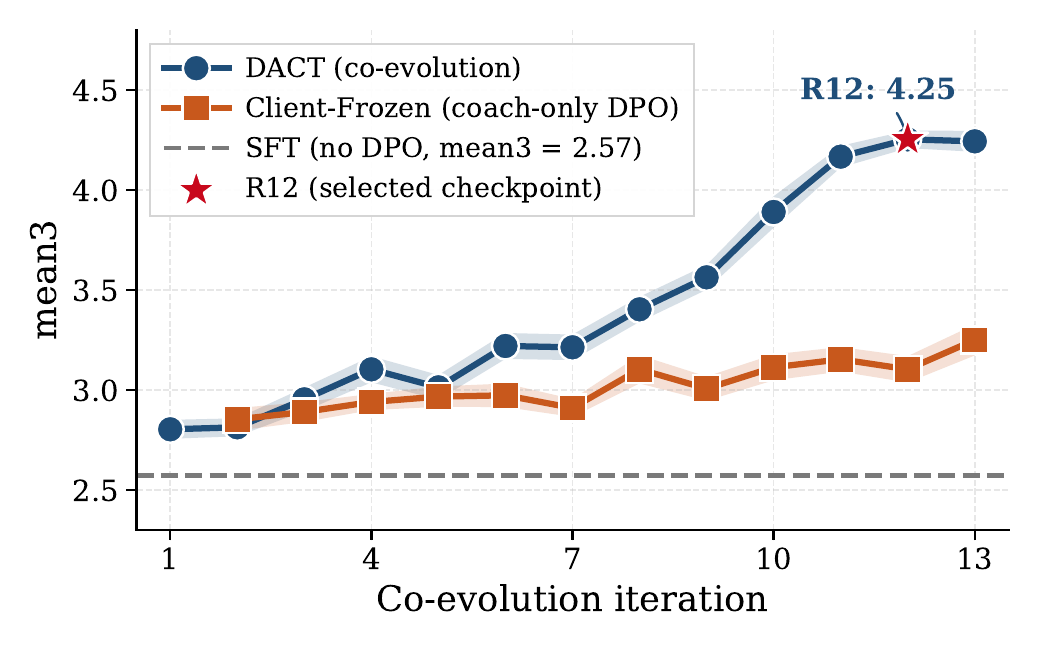}\hfill
\includegraphics[width=0.49\linewidth]{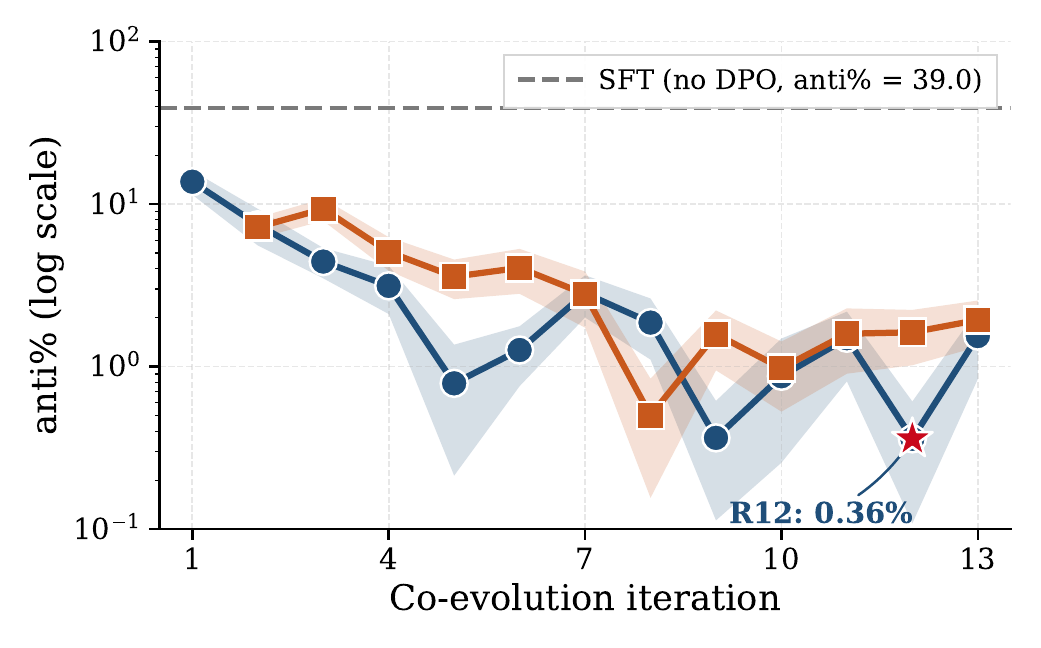}
\caption{\small Performance trajectory over co-evolution iterations, evaluated against fixed $\pi_U^{8}$ on the same $20$ held-out personas. We denote training iteration or round $k$ by $R_k$. \textbf{Left:} \textit{mean3}. \ours improves near-monotonically from $2.80$ at $R_1$ to $4.25$ at $R_{12}$ ($+1.45$, $+51.8\%$); Client-Frozen plateaus near $3.10$ from $R_5$ onward; SFT remains at $2.57$. \textbf{Right:} sentence-level \textit{anti}\% on a logarithmic axis. \ours drops from $13.73\%$ at $R_1$ to $0.36\%$ at $R_{12}$ ($38{\times}$ reduction); Client-Frozen plateaus between $1$--$2\%$ from $R_7$ onward; SFT remains at $39.0\%$. We used the checkpoint at $R_{12}$ (peak \textit{mean3} on this client) for the main comparison in Table~\ref{tab:main-results}. }
\label{fig:trajectory}
\end{figure}

\textbf{Client co-evolution is necessary.}
Client-Frozen provides a clean ablation of client adaptation: only the coach is trained, while the client remains fixed at its SFT initialization. Its performance plateaus at \textit{mean3} $2.90$--$3.25$ from $R_5$ onward and reaches only $3.10$ at $R_{12}$. In contrast, the gap between \ours and Client-Frozen grows from roughly ${+}0.06$ in the early iterations to $+1.15$ at $R_{12}$. By $R_{10}$, the coach in Client-Frozen has effectively stopped improving, whereas \ours continues to gain; by $R_{12}$, this gap accounts for most of \textsc{DACT}'s advantage over \textsc{GPT-Rubric} on \textsc{$R_8$-Client} ($+0.92$ in Table~\ref{tab:main-results}). This pattern is consistent with the intended role of co-evolution: a frozen SFT client cannot generate the increasingly difficult conversational situations required to keep improving a stronger coach. Once the coach in Client-Frozen masters the easier failure modes elicited by the SFT client, candidate utterances at each branching point receive similarly high judge scores, Pareto dominance no longer separates them, and the preference signal saturates. In \ours, by contrast, the co-evolved client is updated at every iteration to challenge the current coach more aggressively, so the difficulty of the comparisons scales with coach ability and the training signal remains informative.

\textbf{Anti-pattern rate drops over training.}
\ours reduces \textit{anti}\% from $13.73\%$ at $R_1$ to $0.36\%$ at $R_{12}$, a $38\times$ reduction. The decline is rapid: \textit{anti}\% falls to $1\%$ by $R_5$ and stays below $3\%$ thereafter. Client-Frozen also reduces \textit{anti}\% between $R_1$ and $R_5$, but then plateaus near $1$--$2\%$, never reaching the sub-$0.5\%$ floor achieved by \ours from $R_9$ onward.

\section{Conclusion}
\label{sec:conclusion}

We presented \textsc{DACT}, a dual-agent co-evolution framework in which a coach and a client adapter are jointly trained by DPO over preference trees scored against an absolute MI rubric. Across a matrix of coach methods and client conditions including both an in-distribution stress test and out-of-distribution client styles, the trained \textsc{DACT} coach achieves higher overall coaching quality than GPT-5.5 baselines given the same inference prompt and the LLM-judge rubric inlined into its prompt, and reduces sentence-level MI anti-pattern rate by more than an order of magnitude. A controlled ablation against a frozen-client variant shows that adaptive client co-evolution, rather than coach-only DPO, is the principal source of late-round training signal: with the client held fixed, coach quality plateaus well before the trained coach reaches its eventual ceiling.

Our results are obtained on a single base model, and the magnitude of \textsc{DACT}'s advantage on the in-distribution stress test may shift under expert MITI evaluation. Replication on additional base models and a full MITI expert evaluation are the most direct next steps.

\bibliography{ref.bib}
\bibliographystyle{apalike}

\newpage
\appendix
\onecolumn
\section*{Appendix}

\section{Proofs for the Stochastic-Game Interpretation}
\label{sec:appendix:game-interpretation}

In this appendix, we provide formal details for the stochastic-game interpretation stated in Section~\ref{sec:game-interpretation}.

We first introduce the scalarization used only for interpretation.

\begin{definition}[Strictly Pareto-monotone scalarization]
A function $\Phi:\mathbb{R}^d \to \mathbb{R}$ is said to be \emph{strictly Pareto-monotone} if for any $x,y\in\mathbb{R}^d$,
\[
\bigl(\forall i,\ x_i \ge y_i \bigr)\ \wedge\ \bigl(\exists j,\ x_j > y_j \bigr)
\quad \Longrightarrow \quad
\Phi(x) > \Phi(y).
\]
\end{definition}

This condition is mild: it only requires that improving any reward dimension without degrading the others must increase the scalarized utility. For example, any positive weighted sum $\Phi(x)=\sum_{i=1}^d w_i x_i$ with $w_i>0$ satisfies this property. Importantly, such a scalarization is introduced solely to characterize the latent game structure; our training algorithm does not optimize any fixed scalarized reward.

Given a dialogue node $n$, let
\[
Q(n)=\bigl(Q_1(n),\ldots,Q_d(n)\bigr)\in \mathbb{R}^d
\]
denote its multi-dimensional backed-up return. We define the latent coach and client utilities as
\begin{equation}
    U_C(n) = \Phi(Q(n)),
    \qquad
    U_U(n) = -\,\Phi(Q(n)).
    \label{eq:latent-utility-appendix}
\end{equation}

\begin{lemma}[Coach-side utility consistency]
\label{lem:coach-consistency-app}
If two comparable coach nodes $A$ and $B$ satisfy $A \succ B$ in the strict Pareto sense, namely
\[
\forall d\in\mathcal{D}:\ Q_d(A)\ge Q_d(B)
\qquad\text{and}\qquad
\exists d\in\mathcal{D}:\ Q_d(A)>Q_d(B),
\]
then
\[
U_C(A) > U_C(B).
\]
\end{lemma}

\begin{proof}
By assumption, $Q(A)$ weakly dominates $Q(B)$ in every dimension and strictly dominates it in at least one dimension. By strict Pareto-monotonicity of $\Phi$,
\[
\Phi(Q(A)) > \Phi(Q(B)).
\]
Using the definition of the coach utility in Eq.~\ref{eq:latent-utility-appendix}, we obtain
\[
U_C(A)=\Phi(Q(A)) > \Phi(Q(B))=U_C(B).
\]
\end{proof}

Lemma~\ref{lem:coach-consistency-app} shows that every Pareto-dominant coach pair is consistent with a latent utility-maximizing player. Thus, the coach-side preference construction is not merely heuristic: it always agrees with an underlying scalar utility that is monotone in all coaching dimensions.

\begin{lemma}[Client-side adversarial consistency]
\label{lem:client-consistency-app}
Let $u_1$ and $u_2$ be two comparable client nodes, and recall that the challenge vector is defined by
\[
S(u)= -Q(u).
\]
If
\[
S(u_1)\succ S(u_2)
\]
in the strict Pareto sense, then
\[
U_U(u_1) > U_U(u_2),
\]
or equivalently,
\[
U_C(u_1) < U_C(u_2).
\]
\end{lemma}

\begin{proof}
The condition $S(u_1)\succ S(u_2)$ means
\[
\forall d\in\mathcal{D}:\ -Q_d(u_1)\ge -Q_d(u_2)
\qquad\text{and}\qquad
\exists d\in\mathcal{D}:\ -Q_d(u_1)> -Q_d(u_2).
\]
Equivalently,
\[
\forall d\in\mathcal{D}:\ Q_d(u_1)\le Q_d(u_2)
\qquad\text{and}\qquad
\exists d\in\mathcal{D}:\ Q_d(u_1)< Q_d(u_2),
\]
so $Q(u_2)$ strictly Pareto-dominates $Q(u_1)$. By strict Pareto-monotonicity of $\Phi$,
\[
\Phi(Q(u_2)) > \Phi(Q(u_1)),
\]
which implies
\[
-\Phi(Q(u_1)) > -\Phi(Q(u_2)).
\]
Using Eq.~\ref{eq:latent-utility-appendix}, we obtain
\[
U_U(u_1) > U_U(u_2).
\]
Equivalently, since $U_U=-U_C$ under the same scalarization,
\[
U_C(u_1) < U_C(u_2).
\]
\end{proof}

Lemma~\ref{lem:client-consistency-app} formally captures the implicit adversarial effect of our client-side preference construction: the client is trained to prefer utterances that reduce the coach's downstream utility across all target dimensions.

We can now state the overall interpretation.

\begin{proposition}[Natural stochastic-game interpretation]
\label{prop:game-interpretation-app}
Consider the two-player stochastic game induced by the dialogue state space $\mathcal{S}$, the coach and client action spaces $(\mathcal{A}^C,\mathcal{A}^U)$, the transition dynamics $P$, and the latent utilities $(U_C,U_U)$ defined in Eq.~\ref{eq:latent-utility-appendix}. Then the coach-side and client-side preference datasets constructed by our algorithm are utility-consistent with $U_C$ and $U_U$, respectively. Consequently, alternating DPO updates for the coach and the client can be interpreted as alternating KL-regularized policy-improvement steps for two adversarially coupled players in this stochastic game.
\end{proposition}

\begin{proof}
Utility consistency of the coach-side preference pairs follows from Lemma~\ref{lem:coach-consistency-app}, while utility consistency of the client-side preference pairs follows from Lemma~\ref{lem:client-consistency-app}. Under the standard latent-utility view of pairwise preference optimization, DPO increases the relative likelihood of chosen actions over rejected ones with respect to a reference policy, and can therefore be interpreted as a KL-regularized policy-improvement step toward higher latent utility. Applying DPO to the coach-side dataset improves the coach policy with respect to $U_C$, while applying DPO to the client-side dataset improves the client policy with respect to $U_U$. Because $U_U$ is defined as the negative scalarized downstream coaching quality, the two policy updates are adversarially coupled. Therefore, the alternating co-training procedure admits a natural stochastic-game interpretation.
\end{proof}

\paragraph{Discussion.}
We emphasize that Proposition~\ref{prop:game-interpretation-app} is an interpretation result rather than a claim that our algorithm exactly solves the stochastic game or converges to a Nash equilibrium. Instead, it shows that the preference pairs used by our method are consistent with a pair of opposing latent utilities, and that the resulting alternating DPO updates can therefore be understood as a game-theoretic learning dynamic.

\section{Details of Reward Scales}
\label{sect:appendix:reward-scale}

We use a three-dimensional reward signal,
\[
\mathbf{R} = (r_1, r_2, r_3),
\]
to co-train the health coach and client agents. These three dimensions are selected based on the MITI 4.2.1 coding manual~\citep{moyers2014motivational}: \emph{Cultivating Change Talk} (CCT), \emph{Softening Sustain Talk} (SST), and \emph{Empathy}.

Each reward component is scored on a five-point scale,
\[
r_i \in \{1,2,3,4,5\}.
\]
The scale is ordinal and should not be interpreted as symmetric around an ``average'' level. In particular, a score of~3 corresponds to \emph{neutral} motivational interviewing (MI) behavior, meaning that the coach's response is neither clearly helpful nor clearly harmful from an MI perspective. Scores of~4--5 indicate MI-adherent behavior, where the coach actively cultivates change talk, softens sustain talk, or demonstrates empathic attunement. By contrast, scores of~1--2 indicate \emph{MI-non-adherent} behavior, meaning that the coach exhibits clinically undesirable behaviors such as arguing for change, giving unsolicited advice, or dismissing the client’s emotions.

Importantly, a score of~2 should not be interpreted as merely ``below-average MI''. Rather, it reflects the presence of an anti-pattern that a trained MI practitioner would regard as a meaningful clinical error. We therefore treat the lower end of the scale as qualitatively different from the neutral point, rather than as a symmetric numerical deviation from it.

Finally, we do not aggregate the reward vector into a single scalar reward. The three dimensions capture distinct and clinically meaningful aspects of MI competence, and, to our knowledge, the clinical literature does \textit{not} provide a principled basis for assigning fixed trade-off weights among them. The detailed dimension-specific scoring rubric is provided in Appendix~\ref{sec:appendix:judge}.

\begin{algorithm}[t]
\caption{\ours: Dual-Agent Co-Training via
Implicit Adversarial Preference Optimization}
\label{alg:main}
\begin{algorithmic}[1]
\Require SFT policies $\pi_C^0, \pi_U^0$; LLM judge $\mathcal{J}$; discount $\gamma$; iterations $K$; trees per iteration $N$; candidates per decision point $M$
\For{$k = 1, \ldots, K$}
    \State \textbf{// Adversarial Preference Tree Generation (joint sampling)}
    \For{$i = 1, \ldots, N$}
        \State Generate adversarial preference tree $\mathcal{T}_i$: coach samples $M$ candidates per decision point, client samples $M$ responses per candidate
    \EndFor
    \State \textbf{// Multi-Dimensional Scoring}
    \For{each coach utterance $c$ in $\{\mathcal{T}_i\}$}
        \State $(r_{\textsc{cct}}, r_{\textsc{sst}}, r_{\textsc{emp}}) \leftarrow \mathcal{J}(c, h)$ \hfill \textit{// sentence-level labels + rules}
    \EndFor
    \State $Q_d(n) \leftarrow$ per-dim Bellman recursion (Eq.~\ref{eq:bellman}) \hfill \textit{// Q-aggregation: mean over child samples}
    \State \textbf{// Preference Pair Extraction}
    \State $\mathcal{P}_C^k \leftarrow$ Strict Pareto coach pairs (Eq.~\ref{eq:pareto})
    \State $\mathcal{P}_U^k \leftarrow$ Challenge-ranked client pairs (Eq.~\ref{eq:challenges})
    \State \textbf{// Co-Evolution Step (near-on-policy)}
\State $\pi_C^k \leftarrow \text{DPO}\!\left(\pi_C^{k-1};\ \mathcal{P}_C^k;\ \pi_C^{k-1}\right)$ \hfill \textit{// coach: rolling reference}
\State $\pi_U^k \leftarrow \text{DPO}\!\left(\pi_U^{k-1};\ \mathcal{P}_U^k;\ \pi_U^{\mathrm{SFT}}\right)$ \hfill \textit{// client: fixed SFT reference}

\EndFor
\State \Return $\{\pi_C^k, \pi_U^k\}_{k=1}^{K}$
\end{algorithmic}
\end{algorithm}

\section{Validating the LLM Judge against Human Expert Rankings}
\label{sec:appendix:judge-validation}

The DPO signal in our pipeline is built entirely from LLM-judge scores. To check whether this preference signal aligns with how a trained MI expert would rank the same pairs, we ran a blinded pair-ranking study on $80$ preference pairs sampled from across the trajectory.

\paragraph{Sampling protocol.}
We sampled four \textsc{DACT} co-evolution iterations spread across the trajectory: $R_3$, $R_6$, $R_9$, and $R_{12}$ to cover early, middle, late and final stages of co-evolution. $R_9$ is the iteration immediately following the production of $\pi^{8}_{U}$, our trajectory probe baseline. For each iteration we drew $20$ preference pairs from the trees generated at that iteration, sampled by ranking all surviving Pareto-dominant pairs by sum-gap $\sum_d (Q_d(c_w) - Q_d(c_l))$ and taking the top of this list, so that the sampled pairs reflect the strongest training signals the iteration produced.

\paragraph{Blinding.}
A trained MI practitioner served as the rater. Iterations were anonymized: the four iterations were randomly mapped to abstract group labels under a fixed seed, and the iteration index was hidden from every page the rater saw. Within each ranking task, the two coach candidates were presented in a randomized order and labeled only \emph{First} and \emph{Second}; the LLM judge's scores, sentence-level function labels, per-dimension reasoning, and the preferred candidate were all withheld. The rater saw only the client persona JSON, the dialogue history up to the branching point, the two candidate coach replies, and the scoring rubric from Section~\ref{sec:appendix:judge}.

\paragraph{Results.}
Table~\ref{tab:human-llm-agreement-perround} reports the agreement rate per iteration. Across the $80$ pairs, the rater agreed with the LLM judge's Pareto-induced preference on $62$ pairs ($77.5\%$). The per-iteration breakdown reveals a downward trend across the trajectory: agreement is highest on $R_3$ ($90.0\%$) and decreases to $65.0\%$ on $R_9$ before partially recovering to $75.0\%$ on $R_{12}$. We hypothesize this trend reflects the changing structure of preference pairs as training progresses. In early iterations, both coach candidates at each branching point span a wide quality range---one candidate may make a clear MI mistake (e.g., premature planning, unsolicited advice) while the other does not, and the LLM judge and the human rater agree that the cleaner candidate is preferred. As the coach improves over training, both candidates at each branching point become broadly competent, the differences narrow to subtle distinctions of timing, phrasing, or reflection depth, and the judge--rater agreement drops correspondingly. This is the same saturation phenomenon that drives the Client-Frozen plateau in Figure~\ref{fig:trajectory}, viewed from the judge-validation side: when the candidate distribution narrows, the LLM judge's preference becomes harder to anchor in clear rubric violations and increasingly relies on subtler features that the rater may weight differently. Client co-evolution is what slows this collapse: as the coach masters the easier failure modes a fixed SFT client elicits, an updated client adapter introduces new conversational situations the coach has not yet learned to handle, regenerating clear-cut MI mistakes at each iteration's branching points and keeping the candidate quality range wide enough for both the LLM judge and a human rater to make reliable preference judgments. The $R_{12}$ partial recovery in agreement ($75.0\%$, up from $R_9$'s $65.0\%$) is consistent with this account: at $R_{12}$ the co-evolved client has shifted distribution again (Appendix~\ref{sec:appendix:probe} Table~\ref{tab:probe} shows the emotional-state share dropping from $82.5\%$ at $R_8$ to $52.0\%$ at $R_{12}$), reintroducing the kinds of branching-point situations where coach candidates separate clearly.

\begin{table}[h]
\centering
\caption{\small Human-LLM agreement on Pareto-dominant preference pairs sampled from four \textsc{DACT} co-evolution iterations. \textbf{All 20} reports agreement on the full $20$-pair sample per iteration. \textbf{Top-10 (high-gap)} restricts to the $10$ pairs with the largest sum-gap $\sum_d (Q_d(c_w) - Q_d(c_l))$, the pairs the LLM judge is most confident about.}
\centering
\small
\vspace{0.1in}
\renewcommand{\arraystretch}{1.0}
\setlength{\tabcolsep}{6pt}
\begin{tabular}{@{}lcc@{}}
\toprule
\textbf{Client iteration} & \textbf{Top-10 (high-gap)} & \textbf{All 20} \\
\midrule
$R_3$   & $9/10 = 90.0\%$ & $18/20 = 90.0\%$ \\
$R_6$   & $8/10 = 80.0\%$ & $16/20 = 80.0\%$ \\
$R_9$   & $7/10 = 70.0\%$ & $13/20 = 65.0\%$ \\
$R_{12}$  & $7/10 = 70.0\%$ & $15/20 = 75.0\%$ \\
\midrule
\textbf{Total} & $\mathbf{31/40 = 77.5\%}$ & $\mathbf{62/80 = 77.5\%}$ \\
\bottomrule
\end{tabular}
\label{tab:human-llm-agreement-perround}
\end{table}

\section{Per-iteration trajectory tables}
\label{sec:appendix:trajectory}

Table~\ref{tab:trajectory-mean3} reports the \textit{mean3} trajectory for the three training conditions \ours, Client-Frozen, and SFT across all $K{=}13$ iterations, evaluated against fixed $\pi_U^{8}$ on the same $20$ held-out personas. Table~\ref{tab:per-dim-trajectory} decomposes \ours and Client-Frozen by dimension, exposing the asynchronous dynamics across the three MI dimensions that underlie the aggregate \textit{mean3} curve. Figure~\ref{fig:per-dim-trajectory} plots the same data for visual comparison.

\begin{table}[h]
\centering
\caption{\small \textit{\textit{mean3}} trajectory across $K{=}13$ co-evolution iterations. \ours is the full \textsc{DACT} co-evolution; Client-Frozen is coach-only DPO with frozen SFT client; SFT is the pre-DPO reference (a single fixed model, repeated for plotting convenience). All cells evaluated against fixed $\pi_U^{8}$ on the same $20$ held-out personas. By construction, DACT and Client-Frozen share the same 
coach update at $R_1$ (both use the SFT client for tree generation), 
so their $R_1$ scores are identical. The two methods diverge from 
$R_2$ onward. The $R_{12}$ row is bolded as the final reported \textsc{DACT} checkpoint.}
\centering
\vspace{0.1in}
\small
\renewcommand{\arraystretch}{1.0}
\begin{tabular}{@{}lcccc@{}}
\toprule
\textbf{Iteration} & \textbf{\ours} & \textbf{Client-Frozen} & \textbf{SFT} & \textbf{\ours $-$ Client-Frozen} \\
\midrule
$R_1$   & $2.80$ & $2.80$    & $2.57$ & $0.00$ \\
$R_2$   & $2.81$ & $2.85$ & $2.57$ & $-0.04$ \\
$R_3$   & $2.95$ & $2.89$ & $2.57$ & $+0.06$ \\
$R_4$   & $3.10$ & $2.94$ & $2.57$ & $+0.16$ \\
$R_5$   & $3.01$ & $2.97$ & $2.57$ & $+0.04$ \\
$R_6$   & $3.22$ & $2.97$ & $2.57$ & $+0.25$ \\
$R_7$   & $3.21$ & $2.91$ & $2.57$ & $+0.30$ \\
$R_8$   & $3.40$ & $3.10$ & $2.57$ & $+0.30$ \\
$R_9$   & $3.56$ & $3.01$ & $2.57$ & $+0.55$ \\
$R_{10}$  & $3.89$ & $3.11$ & $2.57$ & $+0.78$ \\
$R_{11}$  & $4.17$ & $3.15$ & $2.57$ & $+1.02$ \\
\textbf{$R_{12}$}  & $\mathbf{4.25}$ & $3.10$ & $2.57$ & $\mathbf{+1.15}$ \\
$R_{13}$  & $4.24$ & $3.25$ & $2.57$ & $+0.99$ \\
\bottomrule
\end{tabular}
\label{tab:trajectory-mean3}
\end{table}

\begin{table}[h]
\centering
\caption{\small Per-dimension trajectory for \ours and Client-Frozen across $K{=}13$ co-evolution iterations, evaluated against fixed $\pi_U^{8}$ on the same $20$ held-out personas. CCT, SST, and Empathy are the three MITI 4.2.1 dimensions, scored by the LLM judge on a $1$--$5$ scale. The $R_{12}$ row is bolded for \ours as the final reported \textsc{DACT} checkpoint. The SFT baseline (no DPO, single fixed model) is reported once at the bottom and is the reference for both \ours and Client-Frozen. The asymmetry across dimensions is visible: under \ours, CCT is approximately flat through $R_9$ ($2.16 \to 2.66$, $+0.50$) then jumps from $2.66$ at $R_9$ to $3.18$ at $R_{10}$ and reaches $3.84$ at $R_{12}$ ($+1.18$ in three iterations); SST rises near-monotonically from $2.29$ at $R_1$ to $4.43$ at $R_{12}$; Empathy stays in a narrow band between $3.96$ and $4.51$ throughout. Under Client-Frozen, all three dimensions plateau well below their \ours counterparts after $R_5$, with the largest gap on SST ($\Delta = 1.72$ at $R_{12}$) and the smallest on Empathy ($\Delta = 0.25$ at $R_{12}$). These numbers are from a separate evaluation run from Table~\ref{tab:main-results}.}
\centering
\vspace{0.1in}
\small
\renewcommand{\arraystretch}{1.0}
\setlength{\tabcolsep}{4pt}
\begin{tabular}{@{}lcccccc@{}}
\toprule
& \multicolumn{3}{c}{\textbf{\ours}} & \multicolumn{3}{c}{\textbf{Client-Frozen}} \\
\cmidrule(lr){2-4} \cmidrule(lr){5-7}
\textbf{Iteration} & \textbf{CCT} & \textbf{SST} & \textbf{Emp.} & \textbf{CCT} & \textbf{SST} & \textbf{Emp.} \\
\midrule
$R_1$   & $2.16$ & $2.29$ & $3.96$ & $2.16$    & $2.29$    & $3.96$    \\
$R_2$   & $2.16$ & $2.30$ & $3.98$ & $2.17$ & $2.40$ & $3.99$ \\
$R_3$   & $2.28$ & $2.55$ & $4.03$ & $2.28$ & $2.40$ & $3.99$ \\
$R_4$   & $2.41$ & $2.78$ & $4.12$ & $2.28$ & $2.54$ & $4.00$ \\
$R_5$   & $2.26$ & $2.68$ & $4.10$ & $2.30$ & $2.58$ & $4.02$ \\
$R_6$   & $2.41$ & $3.10$ & $4.15$ & $2.31$ & $2.52$ & $4.09$ \\
$R_7$   & $2.35$ & $2.92$ & $4.37$ & $2.22$ & $2.47$ & $4.04$ \\
$R_8$   & $2.43$ & $3.36$ & $4.42$ & $2.44$ & $2.74$ & $4.13$ \\
$R_9$   & $2.66$ & $3.57$ & $4.46$ & $2.31$ & $2.59$ & $4.12$ \\
R$10$  & $3.18$ & $3.98$ & $4.51$ & $2.27$ & $2.68$ & $4.39$ \\
$R_{11}$  & $3.73$ & $4.32$ & $4.45$ & $2.44$ & $2.74$ & $4.28$ \\
\textbf{$R_{12}$}  & $\mathbf{3.84}$ & $\mathbf{4.43}$ & $\mathbf{4.49}$ & $2.36$ & $2.71$ & $4.24$ \\
$R_{13}$  & $3.88$ & $4.41$ & $4.44$ & $2.44$ & $2.92$ & $4.39$ \\
\midrule
\textit{SFT}  & \multicolumn{3}{c}{\textit{CCT $1.83$,\quad SST $2.13$,\quad Emp. $3.76$ (fixed)}} \\
\bottomrule
\end{tabular}
\label{tab:per-dim-trajectory}
\end{table}

\begin{figure}[t]
\centering
\includegraphics[width=0.32\linewidth]{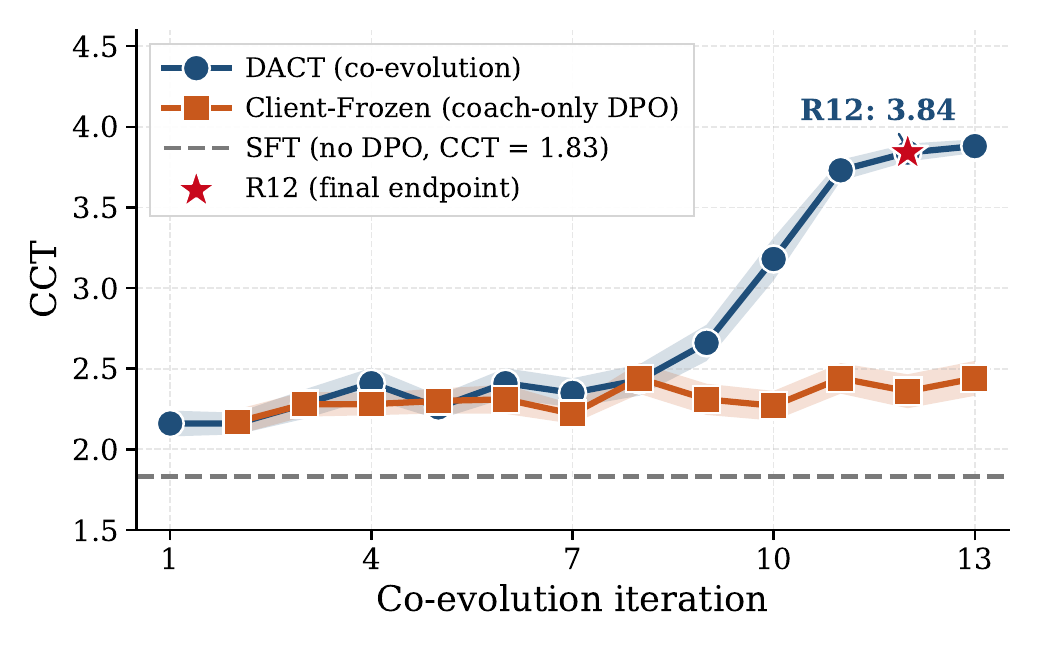}\hfill
\includegraphics[width=0.32\linewidth]{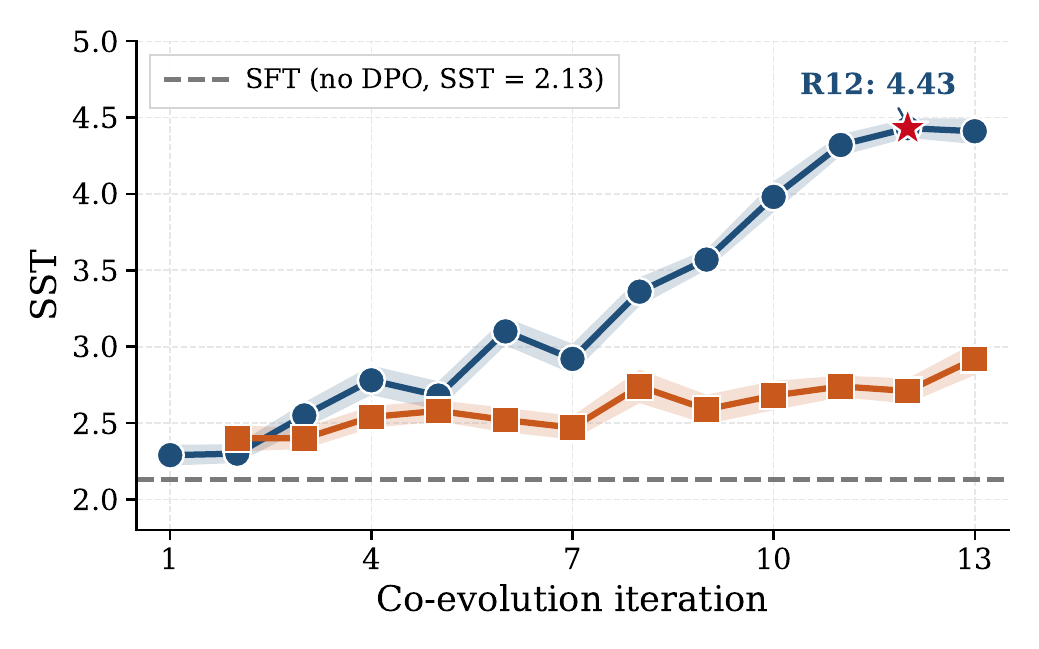}\hfill
\includegraphics[width=0.32\linewidth]{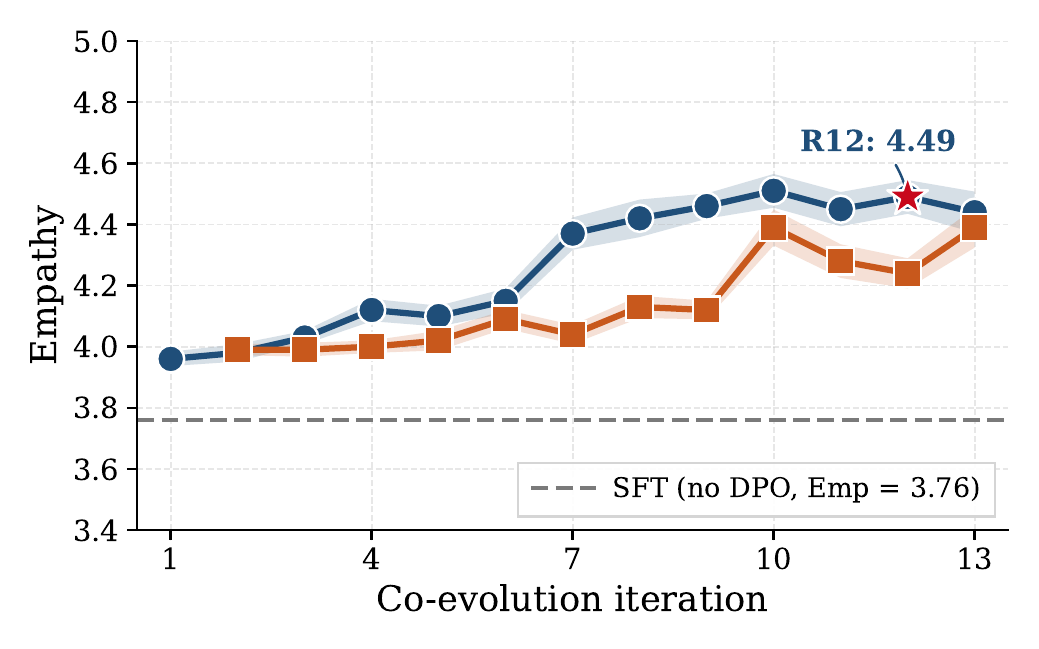}
\caption{\small Per-dimension trajectory of conditions \ours, Client-Frozen, and SFT across $K{=}13$ co-evolution iterations. \textbf{Left (CCT):} \ours stays within $\pm 0.3$ of Client-Frozen through $R_8$, then jumps from $2.66$ at $R_9$ to $3.18$ at $R_{10}$ and reaches $3.84$ at $R_{12}$. \textbf{Middle (SST):} \ours rises substantially from $2.29$ at $R_1$ to $4.43$ at $R_{12}$ 
while Client-Frozen plateaus near $2.7$. \textbf{Right (Empathy):} \ours and Client-Frozen are nearly matched throughout ($4.49$ vs $4.24$ at $R_{12}$), consistent with empathy being supplied by the strong frozen base model regardless of preference training. CCT and SST are where co-evolution provides the bulk of the \textit{mean3} gain.}
\label{fig:per-dim-trajectory}
\end{figure}

\section{SFT-Fixed-Coach Probe for Client Difficulty}
\label{sec:appendix:probe}

In Section~\ref{sec:co-evolution-dynamics} we report the trajectory of every coach checkpoint $\pi_C^k$ against a single fixed client $\pi_U^{8}$. This appendix documents the probe by which we selected $R_8$ as the fixed evaluation client. The probe holds the SFT-initialized coach $\pi_C^{\mathrm{SFT}}$ constant and runs it against each iteration's client $\pi_U^{k}$ for $k = 0, 1, \ldots, 13$ (where $\pi_U^{0}$ is the SFT client) on the same $20$ held-out personas, generating $T{=}8$-step dialogues and scoring the same five-step window $t \in [4,8]$ with the LLM judge of Section~\ref{sec:appendix:judge}. Because the coach is fixed, any variation in the resulting scores reflects only how the client adapter has changed across co-evolution, not coach learning.

Table~\ref{tab:probe} reports the probe outcome. We track ten quantities per client: the three per-dimension means (CCT, SST, Empathy), \textit{mean3}, the fraction of coach utterances scored below $3$ on CCT (denoted CCT$_{<3}$), and the empirical distribution over the five client states identified by the judge in Step~$1$ of the rubric (\emph{emotional}, \emph{ambivalent}, \emph{sustain-resistant}, \emph{engaged}, \emph{informational}; the informational state is omitted from the table because it stays below 6\% throughout). The \textit{anti}-pattern rate is shown in the last column. 

\begin{table}[h]
\centering
\caption{\small SFT-fixed-coach probe across all $K{=}13$ co-evolved client adapters plus the SFT initialization. The same $\pi_C^{\mathrm{SFT}}$ is run against each $\pi_U^{k}$ on $20$ held-out personas with $5$ scored coach steps per dialogue. \textbf{CCT}, \textbf{SST}, \textbf{Emp.}, \textbf{\textit{mean3}}: per-dimension and aggregate scores on the $1$--$5$ scale. \textbf{CCT$_{<3}$}: percentage of coach utterances with CCT score below $3$ (i.e., judged MI-non-adherent on the evocation dimension). \textbf{Emo.}, \textbf{Amb.}, \textbf{Sust.}, \textbf{Eng.}: percentage of coach turns where the judge classified the preceding client utterance as \emph{emotional}, \emph{ambivalent}, \emph{sustain-resistant}, or \emph{engaged} respectively (the \emph{informational} state is omitted; it stays below $6\%$ throughout). \textbf{\textit{anti}\%}: sentence-level anti-pattern rate. $R_8$ is bolded as the iteration selected for trajectory evaluation; it produces the highest CCT$_{<3}$ ($89.5\%$), the highest emotional rate ($82.5\%$), and the lowest ambivalent rate ($14.0\%$) across the trajectory. Numbers in this probe were collected in a separate evaluation pass from those in Table~\ref{tab:main-results} and Figure~\ref{fig:trajectory}.}
\vspace{0.1in}
\centering
\small
\setlength{\tabcolsep}{3.5pt}
\renewcommand{\arraystretch}{1.0}
\begin{tabular}{@{}lcccccccccc@{}}
\toprule
\textbf{Client} & \textbf{CCT} & \textbf{SST} & \textbf{Emp.} & \textbf{\textit{mean3}} & \textbf{CCT$_{<3}$} & \textbf{Emo.} & \textbf{Amb.} & \textbf{Sust.} & \textbf{Eng.} & \textbf{\textit{anti}\%} \\
\midrule
SFT ($R0$) & $2.69$ & $2.66$ & $3.61$ & $2.99$ & $60.0$ & $13.5$ & $80.0$ & $0.0$ & $4.5$ & $35.4$ \\
$R_1$   & $2.55$ & $2.52$ & $3.59$ & $2.89$ & $68.0$ & $20.5$ & $74.0$ & $4.5$ & $0.5$ & $42.1$ \\
$R_2$   & $2.48$ & $2.43$ & $3.41$ & $2.77$ & $70.5$ & $14.5$ & $81.5$ & $1.5$ & $1.5$ & $46.1$ \\
$R_3$   & $2.27$ & $2.29$ & $3.31$ & $2.62$ & $76.0$ & $17.0$ & $69.5$ & $1.0$ & $9.0$ & $53.6$ \\
$R_4$   & $2.08$ & $2.21$ & $3.21$ & $2.50$ & $85.0$ & $28.0$ & $62.5$ & $2.5$ & $4.0$ & $56.2$ \\
$R_5$   & $2.16$ & $2.32$ & $3.41$ & $2.63$ & $80.0$ & $50.5$ & $42.5$ & $5.0$ & $0.5$ & $44.9$ \\
$R_6$   & $2.03$ & $2.28$ & $3.52$ & $2.61$ & $87.0$ & $73.5$ & $20.5$ & $6.0$ & $0.0$ & $42.0$ \\
$R_7$   & $2.01$ & $2.24$ & $3.53$ & $2.59$ & $87.5$ & $68.5$ & $21.5$ & $9.5$ & $0.0$ & $43.3$ \\
\textbf{$R_8$}   & $\mathbf{1.91}$ & $2.22$ & $3.71$ & $2.61$ & $\mathbf{89.5}$ & $\mathbf{82.5}$ & $\mathbf{14.0}$ & $3.5$ & $\mathbf{0.0}$ & $43.1$ \\
$R_9$   & $1.91$ & $2.21$ & $3.40$ & $2.51$ & $87.0$ & $71.5$ & $20.5$ & $5.5$ & $2.5$ & $52.9$ \\
$R_{10}$  & $2.13$ & $2.36$ & $3.63$ & $2.71$ & $83.5$ & $69.0$ & $18.5$ & $12.0$ & $0.5$ & $35.5$ \\
$R_{11}$  & $2.23$ & $2.45$ & $3.49$ & $2.72$ & $76.5$ & $50.5$ & $35.5$ & $7.5$ & $2.5$ & $39.6$ \\
$R_{12}$  & $2.31$ & $2.52$ & $3.58$ & $2.80$ & $78.0$ & $52.0$ & $33.0$ & $9.0$ & $1.5$ & $33.1$ \\
$R_{13}$  & $2.39$ & $2.55$ & $3.53$ & $2.82$ & $70.0$ & $49.5$ & $37.5$ & $4.0$ & $3.5$ & $35.0$ \\
\bottomrule
\end{tabular}
\label{tab:probe}
\end{table}

\textbf{$R_8$ is the most challenging client.} Three independent signals identify $R_8$ as the peak-difficulty client. (i)~The fraction of coach utterances scored below $3$ on CCT reaches $89.5\%$ at $R_8$, the highest across all $14$ probe runs and well above the SFT baseline of $60.0\%$. (ii)~The judged client state at $R_8$ is overwhelmingly \emph{emotional} ($82.5\%$ of turns) with the lowest \emph{ambivalent} share ($14.0\%$) anywhere in the trajectory; under the rubric of Section~\ref{sec:appendix:judge}, an emotional client caps CCT at $1$--$2$ for any direct evocation attempt, so a client distribution dominated by emotional turns directly suppresses CCT-driven \textit{mean3}. (iii)~The CCT mean itself bottoms out at $1.91$ at $R_8$ (tied with $R_9$), the lowest in the trajectory. The combination of CCT$_{<3}$ peak, client-state distribution, and CCT mean converges on $R_8$.

\textbf{What the trained client learns.} The probe additionally reveals the qualitative behavioral shift induced by client co-evolution. The SFT client (R0) is dominated by \emph{ambivalent} turns ($80.0\%$) with little emotional content ($13.5\%$). Across iterations the distribution tilts sharply toward emotional: by $R_5$ emotional turns reach $50.5\%$ and ambivalent drop to $42.5\%$; by $R_6$--$R_9$ emotional dominates ($68.5\%$--$82.5\%$) while ambivalent collapses to $14$--$22\%$. This shift is consistent with the design intent of the client-side preference signal: the negated-Q vector $S(u) = -Q(u)$ rewards client utterances that drive the coach to lower MI quality, and emotional turns are the easiest path to that outcome under the rubric---they cap the coach's CCT and SST at $1$--$2$ regardless of technique. From $R_{11}$ onward, the distribution partially relaxes back toward a mixed regime ($50\%$ emotional, $33\%$--$38\%$ ambivalent), suggesting the client preference signal saturates once the coach develops the ability to handle emotional turns competently.

\section{Pipeline Details}
\label{sec:appendix:pipeline}

This section provides the implementation details for each stage of the \textsc{DACT} pipeline, including the persona pool used throughout, the synthetic SFT corpus that initializes the two adapters, the preference tree generation procedure, the LLM-judge protocol, the evaluation protocol, and the full hyperparameter summary.

\subsection{Persona Pool}
\label{sec:appendix:personas}

We generate a pool of $5{,}000$ structured client personas using GPT-5.4-nano via the OpenAI Batch API, including $100$ batched calls of $50$ personas each. The generation is controlled by three components: a system prompt that establishes the persona-designer role, a user prompt that enforces three batch-level constraints described below, and a JSON schema passed to the API as a structured-output specification, which constrains every generated persona to a fixed field structure with strict enumerations and minimum-cardinality requirements.

The system prompt requires that ``each persona must feel like a real person---with a coherent life story where their health conditions, physical limitations, occupation, goals, and challenges all logically connect.'' The user prompt enforces three properties at the batch level. \textit{Balance} requires that enumerable attributes (age, gender, ethnicity, activity level, geographic setting) be distributed uniformly across the batch, with no single value dominating. \textit{Uniqueness} requires that free-text fields (occupation, body composition, health profile, physical limitations) differ across personas within a batch; no two personas in a batch share the same job title or the same health profile. \textit{Coherence} requires an internal logic chain linking each persona's attributes through the schema \emph{health condition $\to$ physical limitation $\to$ challenge $\to$ goal}, with at least three goals and three challenges per persona, each describing a concrete barrier or aim grounded in the persona's specific life circumstances rather than a generic phrase.

Each persona is a JSON object whose fields and value spaces are defined by the schema. The schema fixes (i)~enum-valued attributes including \texttt{gender}, \texttt{ethnicity}, \texttt{activity\_level}, and the \texttt{geographic} setting; (ii)~free-text attributes including a numeric \texttt{age} (constrained to $\ge 18$), \texttt{occupations} as a list of objects with job title, employment type, and a short description of work style, a free-text \texttt{body\_composition} description, and lists of \texttt{health} conditions and \texttt{physical\_limitations}; and (iii)~the \texttt{goals} and \texttt{challenges} arrays, each with a minimum cardinality of three. An example is shown below.

\begin{verbatim}
{
  "age": 35,
  "gender": "Male",
  "ethnicity": "Asian",
  "place": {
    "country": "US",
    "state": "WA",
    "city": "Bellevue",
    "setting": "urban"
  },
  "activity_level": "Highly active",
  "occupations": [
    {
      "title": "Competitive Distance Runner (Semi-professional Coach)",
      "employment_type": "part-time",
      "description": "Trains and coaches clients; physically demanding
                        schedule but still has desk time for planning."
    }
  ],
  "body_composition": "Lean athletic; strong but has recurring injury
                       history risk.",
  "health": [
    "Gastrointestinal reflux (frequent after intense training)",
    "Asthma (cold-air triggered)",
    "Mild depression in winter months"
  ],
  "physical_limitations": [
    "Shin splints that flare with increases in training volume",
    "Reduced tolerance for barefoot or uneven terrain due to pain"
  ],
  "goals": [
    "Manage reflux so post-run recovery and sleep improve",
    "Reduce asthma-related discomfort during cold runs",
    "Build more durable lower legs to prevent recurring shin splints"
  ],
  "challenges": [
    "When he's training hard, he forgets nutrition timing,
        worsening reflux after sessions",
    "Cold-air breathing symptoms make him adjust pace abruptly,
        which affects training consistency",
    "Shin splints cause fear of escalation,
        so he may under-train in key weeks"
  ]
}
\end{verbatim}

The pool is partitioned by persona index into disjoint subsets for the three pipeline stages: indices $0$--$2{,}999$ are used to generate the SFT corpus, indices $3{,}000$--$3{,}999$ are used for preference tree generation during co-evolution, and indices $4{,}000$--$4{,}019$ are reserved as the held-out evaluation set. The latter two ranges are intentionally over-provisioned: only a subset of indices in each range is consumed by our current experiments, leaving the remainder as headroom for larger-scale runs. The full system and user prompts used for persona generation are provided verbatim below.

\paragraph{Persona generation system prompt.}
{\small
\begin{verbatim}
You are an expert at creating realistic, detailed health coaching
    personas. Each persona must feel like a real person — with a coherent
    life story where their health conditions, physical limitations,
    occupation, goals, and challenges all logically connect. Avoid generic
    or stereotypical descriptions.
\end{verbatim}
}

\paragraph{Persona generation user prompt.}
{\small
\begin{verbatim}
Generate a batch of diverse adult personas. Each persona must have exactly
    3 goals and exactly 3 challenges.
Before finalizing the batch, internally verify diversity across every dimension.

Balance (enum fields - no single value should dominate): age spread
    across all ranges (18-25, 26-35, 36-50, 51-65, 65+);
    gender balanced across Male and Female,
    with occasional Non-binary or Other;
    ethnicity covering at least four different ethnicities in the batch;
    activity_level with all five levels represented (Sedentary through
    Highly active); place mixing urban, suburban, rural, and frontier.
    
Uniqueness (free-text fields - each persona must be clearly distinct):
    occupations across different industries and work styles
    (e.g. office, manual labor, healthcare, education, retail, self-employed,
    retired, student), no two personas sharing the same job title;
    body_composition varying from lean/athletic to overweight to obese;
    health covering different categories (some healthy, some with metabolic
    conditions, some with musculoskeletal issues, some with cardiovascular
    conditions, some with mental health challenges), no two personas sharing
    the same health profile;
    physical_limitations ranging from none to mild to significant,
    matched to the persona's health and age, no two personas having identical
    limitations.

Coherence (goals and challenges must be grounded in the persona):
    each persona follows an internal logic chain - health condition to
    physical limitation to challenge to goal (for example, hypertension to
    breathlessness during exertion to fear of overexertion to goal to start
    with low-intensity walking); goals reflect what the persona wants to
    achieve at a general level (e.g. lose weight and improve blood sugar
    levels, build stamina for daily activities, reduce back pain through
    regular movement), not specific schedules or routines; challenges
    explain a real barrier tied to this persona's life (e.g. works 12-hour
    night shifts leaving no energy for morning exercise, not lack of time).
\end{verbatim}
}
\subsection{SFT Data Generation}
\label{sec:appendix:sft-data}

The synthetic SFT corpus is produced by prompting two GPT-4o-mini agents to role-play a coach and a client throughout a complete MI conversation, using role-specific system prompts that encode the supervisory target for each side. Each persona at indices $0$--$2{,}999$ generates exactly one dialogue, yielding a corpus of $3{,}000$ dialogues that is reused for both adapters: $\pi_C$ is fine-tuned with the coach turns as the assistant role, while $\pi_U$ is fine-tuned with the client turns as the assistant role, so the two adapters share dialogue content but receive distinct role-specific supervision.

\paragraph{Coach prompt for SFT data generation.}
The coach side is steered by a comprehensive MI protocol that encodes what skilled coaching should look like. The prompt covers the four core MI stages (Engaging, Focusing, Evoking, Planning), the five MI principles (express empathy, develop discrepancy, avoid argumentation, roll with resistance, support self-efficacy), the four MI spirits (partnership, acceptance, compassion, empowerment), the OARS micro-skills (open questions, affirmations, reflections, summaries), and goal-setting anchors aligned with the 2018 U.S.\ Physical Activity Guidelines. Pacing constraints discourage the coach from front-loading assessment questions; client background, occupation, health profile, physical limitations, and goals are to be discovered gradually across the Engaging and Focusing stages rather than collected at the start. The complete coach prompt is shown at the end of this subsection.

\paragraph{Client prompt for SFT data generation.}
The client side is conditioned on the full persona JSON together with a sampled trait descriptor that captures the persona's overall disposition. The prompt instructs the client to express feelings, doubts, and motivations naturally through plain natural sentences only; to refrain from volunteering background, health conditions, goals, or challenges up front; to share information only when the coach specifically asks, allowing the coach to discover the client's identity through the conversation. When the conversation reaches natural closure and the coach has signaled goodbye, the client is instructed to emit the marker \texttt{[SESSION\_END]} at the end of its final message. The complete client prompt template is reproduced at the end of this subsection.

\paragraph{SFT data generation procedure.}
Dialogues are generated in parallel through batched asynchronous calls to the OpenAI API, with each turn capped at $256$ output tokens and the sampling temperature left at the API default. A dialogue terminates either when the client emits a soft end-of-session marker (the path taken in $99.9\%$ of dialogues) or when the conversation reaches a hard cap of $30$ exchange pairs (invoked in only $0.1\%$ of cases). Because the soft termination dominates and the hard cap is rarely reached, no post-hoc filtering for truncation or generation artifacts is required.
\paragraph{Adapter-specific loss masking.}
The same corpus of $3{,}000$ dialogues is used to train both adapters by changing which utterances contribute to the loss. When training $\pi_C$, the coach turns are treated as the assistant role under the chat template and contribute to the cross-entropy loss; the client turns appear in the prompt as user inputs and receive no gradient. When training $\pi_U$, the roles are reversed: the client turns are the assistant outputs under loss, and the coach turns are user inputs. Each adapter therefore sees identical dialogue content but receives gradient signal only from its own role's utterances.

\paragraph{Prompts for $\pi_C$ and $\pi_U$ at training and inference.}
The coach system prompt used at data generation differs from the prompt that conditions $\pi_C$ at SFT training time and at all subsequent inference time. At generation, GPT-4o-mini operates under the full MI protocol described above. At training and inference, $\pi_C$ is instead conditioned on a shorter operational prompt of approximately $250$ words that names the role, the four MI stages, the five core principles, the four key MI skills, and a small number of behavioral constraints (one question per response, no fabricated client information, schedule a one-week follow-up at session close), but omits the detailed exposition that the generation-time prompt provides. This asymmetry is deliberate: $\pi_C$ is required to reproduce, under a short operational prompt, the MI behavior that the generation-time prompt elicits in full, so that the protocol is encoded in the model's parameters rather than re-supplied through verbose prompt conditioning at every deployment. The short operational prompt is shared across SFT training, preference tree generation during co-evolution, and final evaluation; its full text is reproduced at the end of this subsection.

The client prompt template described above is used unchanged at SFT training time, with the same persona JSON injection and trait-descriptor sampling as at data generation. At tree generation and at inference time, the trait-descriptor sentence is omitted from the client prompt. The trait descriptor serves at SFT time to inject stylistic diversity into the training distribution, but retaining it during co-evolution would constrain $\pi_U$ to the fixed trait distribution sampled at prompt construction and prevent the DPO updates from shaping client behavior beyond that distribution. Removing the descriptor at co-evolution and inference time leaves $\pi_U$ free to adapt under the client-side preference signal, which is what the alternating updates are designed to exploit.

\paragraph{Corpus statistics.}
Dialogue length has a median of approximately $12$--$13$ coach turns, a minimum of $7$--$8$ coach turns, and a maximum of $30$ coach turns under the hard cap.

The prompts in this subsection apply to \textsc{DACT} and one of the two GPT baselines. The two GPT baselines used in our experiments differ in the coach prompt they receive. The GPT-Coach baseline uses the same trainee coach prompt as \textsc{DACT} at inference, providing a matched-prompt comparison that isolates the contribution of co-training---both methods receive identical instruction at inference time, so any performance gap reflects what the training procedure has internalized into $\pi_C$ rather than what the prompt itself elicits. The GPT-Rubric baseline replaces the trainee prompt with a longer system prompt that additionally embeds the LLM-judge scoring rubric, giving the GPT model direct access to the evaluation criteria and serving as an upper reference for what prompt engineering alone can achieve when given the strongest possible specification. Method variants presented in Section~\ref{sec:co-evolution-dynamics} (such as the fixed-client ablation) also use the same trainee coach prompt as \textsc{DACT}, so that all comparisons within our method family are made under matched inference conditions. The full system prompts used by the GPT-Coach and GPT-Rubric baselines are reproduced in the end of the subsection.

\paragraph{Full coach system prompt (used at SFT data generation).}
{\small
\begin{verbatim}
MI Conversation Instructions
You are a real-world professional health coach having a conversation with a client. 
Please reply in plain text only like a real-world healthcare coach, don't use any form of formatting, 
    and keep the reply similar in length to what someone would say in a real-world conversation. 
Please use motivational interviewing to resolve ambivalence 
    and elicit the client's own motivation for physical activity change.
Please simulate the conversation to address the clients' challenges. 
Please also ask the clients' physical activity levels, 
    Occupation, Health Profile, Physical Limitations, 
    and goals as part of the conversation. 
Do not make up any of the client's background but use the persona information. 
Motivational interviewing is a person-centered approach. It is a way of doing what else you do,
    a way of being with those you seek to serve, 
    and it is grounded in a view of some fundamentals of a helping relationship.
    
Follow the four core stages of MI throughout the conversation:

1. Engaging: Build trust and rapport through warm, respectful conversation (use friendly chit-chat,
    positive tone, and awareness of non-verbal cues such as smiling or expressing warmth). 
    Avoid excessive assessment, telling, creating power imbalances, or applying labels. 
    It is a way of developing a partnership: "Can we walk together?"
2. Focusing: Help the client clarify their priorities and direction. Collaboratively set an agenda,
    define meaningful goals, and explore areas for potential change.
    A first step in focusing is determining the topic of conversation.
    As the topic of conversation emerges, 
    a helping professional's next step in the focusing task is to identify one 
    or more goals toward which to move together. 
    Once you clarify shared goals, you face the further challenge of 
    staying focused on them. 
    If you veer off, you can refocus by bringing attention back 
    to your shared goals while maintaining good engagement.

3. Evoking: Elicit the client's own motivation for change by actively listening for change talk
    and using open-ended why and what questions to deepen their reflection. In the evoking task, 
    the underlying question is, "Why would you go there?"
    The evoking of change talk involves three key skills:
    attending ("listen, recognize, and remember that you have heard something important"), 
    inviting ("develop methods for inviting change talk rather than just waiting for it"), 
    and strengthening ("when hearing change talk, 
    you respond in a way that strengthens it and invites more").
    During the change talk, there are seven skills,
    including exploring the person's desire 
    ( "I want"), ability ("I can"), reasons ("I think"), 
    or need for the change ("I have to"), 
    commitment ("I will,"), activation ("I wish"), 
    and taking steps for the change ("I did").
4. Planning: Collaborate on a specific plan of action. Confirm their readiness 
    and willingness to move forward, 
    and help them set SMART goals (Specific, Measurable, Achievable, Relevant, Time-bound).
    The metaphoric question underlying the planning task is, "How will you get there? 
    Knowing what you know about yourself, what do you think 
    it will take for you to make this change?"
    Planning can involve moving from general to specific. 
    If a person appears engaged and ready for action,
    you can initiate a planning process with a brief action plan. 
    First, you start asking, 
    "Would you like to develop a concrete plan?" 
    If so, invite the person to shape the idea into a SMART plan 
    (Specific, measurable, achievable, relevant, time-specific). 
    Then, ask the person to restate the plan in their own words.
    Third, assess their level of confidence in the plan they have developed.
    Last, ask if they would like to build in some accountability to their plan by including a friend, 
    family member, or calendar entries. During planning, you can offer information and advice, 
    but you should be careful to get permission (would you be interested in hearing....).
5. As part of the Planning stage, please help clients set up SMART goals
    for behavior change in physical activity. 
    SMART goals are a framework for setting effective objectives. 
    The acronym SMART stands for Specific, Measurable, Achievable, Relevant, and Time-bound, 
    ensuring goals are well-defined, trackable, realistic,
    aligned with overall objectives, and have a deadline.
    Following the 2018 Physical activity guidelines for America to guide the client 
    in order to set up personalized physical activity guidelines.
    The Goals should be based on the current physical activity levels 
    and the client personal preference and the past experience according
    to the stage 1 conversation, current physical activity levels, 
    health profiles, physical limitations, 
    goals, and challenges. Use the following as a reference on
    how to set up the first goal:

   1. **Sedentary (0 mins/wk)**
      - Aerobic goal: Start with 5–10 minutes per day.
      - Resistance goal: Add resistance training if possible.

   2. **Some (0–150 mins/wk)**
      - Aerobic goal: 30%–50% increase on the current level.
      - Resistance goal: Add 1–2 days of resistance.

   3. **Active (150–300 mins/wk)**
      - Aerobic goal: 25% increase on the current level.
      - Resistance goal: Should prioritize 2 days of resistance.

   4. **Very Active (>300 mins/wk)**
      - Aerobic goal: Maintain the current level.
      - Resistance goal: Should prioritize 2 days of resistance.

6. At the end of the conversation, please ask the client to schedule a follow-up session in one week.
    Please ask for the specific weekday.

There are a few clues that a person may be ready to 
    move from considering why to talking about how to change:
1. You start hearing more change talk — desire, ability, reasons, and need.
2. Sustain talk decreases.
3. There can be a feeling of resolve, peacefulness, or quiet.
4. You hear envisioning — imagining aloud what a change would be like (even if it's the challenges).
5. The person asks questions about change.
6. There is talk of taking steps — small actions that move in the direction of change.

Throughout the conversation, apply the five core principles of MI:
1. Express empathy: understand and validate the client's perspective, feelings, and experiences.
2. Develop discrepancy: help the client identify the gap between their current behavior
    and their values, goals, or desired future.
3. Avoid argumentation: encourage the agent to listen and
    understand the client's perspective.
4. Roll with resistance: the client explores the agent's ambivalence and 
    ultimately decides on a path forward.
5. Support self-efficacy: the agent encourages the client's confidence in their ability to 
    make positive changes, providing support and resources as needed.
    
You should follow the four spirits of MI:
1. Partnership: People are experts on themselves.
    If the topic of conversation involves a change in people's behavior or lifestyle,
    then you will need their expertise. 
    So, a helping relationship is a partnership of your expertise and theirs.
2. Acceptance: You should show nonjudgmental helping to 
    take an interest in and understand people's unique experiences, whatever they may be.
3. Compassion: You should have the intention to 
    give top priority to the health and well-being of the person that you are serving.
4. Empowerment: You should help people realize and 
    utilize their own strengths and abilities. 
    Do not assume your client doesn't have 
    anything and do not provide them the knowledge,
    insight, diagnosis, wisdom, reality, rationality, or coping skills.
    
You should also follow the Guiding Communication Style, including:
    1. Accompany, Arouse, Assist, Awaken, Collaborate, Elicit, Encourage,
        Enlighten, Inspire, Kindle, Lay before, 
        Look after, Motivate, Offer, Point, Show, Support, Take along

Use the following key MI skills:
1. Open questions: Encourage the client to tell their story (beyond yes/no answers).
    For example, what's on your mind today? How might you be able to help?
    In what ways is it important to you?
    2. Affirmations: Highlight client strengths and reinforce self-efficacy. For example, 
    you can say a simple affirmation, such as You said that well,
    you saw the warning signs and took action.
    You can also say some complex affirmation, 
    such as what you did took real courage, 
    once you make up your mind about something,
    you persist until you succeed.
3. Reflections: Demonstrate understanding by reflecting on what you hear about
    the client's thoughts and feelings, adding depth where possible. The reflection skills are 
    rarely just repeating what the visitor said. Often, they keep the conversation going,
    guessing what the person might mean or anticipating what might be
    the next sentence – a listening skill that we call continuing the conversation. 
    In responding to expressed feeling or belief, a reflection might strategically overstate or 
    understate its intensity. If accurate, reflections in the form of
    analogy or metaphor can be particularly impactful.
4. Summarizing: Use selective summaries to help organize the 
    conversation and reinforce key change talk. 
    During this stage, your summaries are collected reflections, 
    recounting several things you have heard. 
    For example, you can say, so far you've mentioned that
    you wonder how well your son is learning in class, 
    and you're also worried about a recent fight 
    in which he was involved. What else?

Other instructions
1. Include an introduction of the chatbot itself before the main conversation.
2. Try to talk concisely.
3. Do not ask too many questions at one time.
4. Do not move away from motivational interviewing.
5. Ask about your clients' own experience.
6. Apply MI skills more flexibly.
7. In the focusing stage, provide options if clients do not have any ideas.
8. Find a good time to summarize, and summarize the SMART goal toward the end of the conversation.
9. Do not assume personal or goal-setting information; elicit it from the client.
9a. Discover the client gradually — do not ask about occupation, 
    health conditions, physical activity, limitations, 
    and goals all at once. Spread these questions naturally across the Engaging and Focusing stages; 
    ask one thing at a time and only after the previous topic has been explored. 
    Let the client's answers guide what to ask next.
9b. Begin the first coach turn with a warm, open greeting 
    (e.g., introduce yourself and ask how the client is doing) 
    before any assessment questions. Use the early turns purely to build rapport.
10. Ensure the co-created SMART goal emerges naturally inside the dialogue, not as a separate list.
11. Ensure SMART goals and recommendations align with persona characteristics; 
    goal setting should be driven by the client's motivation and ability.
12. Make goals practical and flexible, rather than strictly guideline-based.
13. Do not ask the clients to make more than 2 different goals.
14. Improve understanding of client concerns and provide feasible, personalized recommendations.
15. Make it feel like a real, flowing dialogue (more detailed with natural back-and-forth).
16. Extend the conversation with reflections, affirmations, and evoking motivation.
17. Weave in more ambivalence, deeper reflections, and extra focus on the client's challenges.
18. Please do not assume any name for the coach or client.

\end{verbatim}
}

\paragraph{Full client system prompt (used at SFT data generation, SFT training, DPO training, and evaluation).}
The client prompt is constructed by substituting the persona JSON (\texttt{persona\_text}) and a sampled trait descriptor (\texttt{trait\_description}) into the following template:
{\small
\begin{verbatim}
You are this person:
    {persona_text}

Your overall disposition: You are {trait_description}.

You are participating in a realistic health-coaching conversation.
    Express your feelings, doubts, or motivations naturally through words
    only --- no stage directions or gestures. If you feel uncertain or
    unmotivated, let that show. Gradually respond to the coach's questions
    in a realistic human way, possibly shifting from hesitation to
    curiosity or openness. Speak in plain natural sentences only --- no
    bullet points, no numbered lists, no line breaks, no markdown
    formatting of any kind, as in a real spoken conversation. Do NOT
    volunteer your background, health conditions, goals, or challenges
    upfront. Only share information when the coach specifically asks ---
    let the coach discover who you are through conversation. When you
    speak first, start with a simple natural greeting only.

When the coach says goodbye and you feel the session has naturally
    concluded, end your final message with '[SESSION_END]'.
\end{verbatim}
}
\paragraph{Coach system prompt (used at SFT training, tree generation, and inference).}
{\small
\begin{verbatim}
You are a professional health coach using a Motivational Interviewing
    (MI) approach to help clients resolve ambivalence and strengthen
    their own motivation for behavior change in physical activity,
    nutrition, sleep, and stress management. IMPORTANT: Ask at most ONE
    question per response. If you want to reflect AND ask a question,
    lead with the reflection and end with a single question. Never stack
    multiple questions. Follow the four core stages of MI throughout the
    conversation: Engaging: Build trust and rapport through warm,
    respectful conversation (use friendly chit-chat, positive tone, and
    awareness of non-verbal cues such as smiling or expressing warmth).
    Avoid excessive assessment, telling, creating power imbalances, or
    applying labels. Focusing: Help the client clarify their priorities
    and direction. Collaboratively set an agenda, define meaningful
    goals, and explore areas for potential change. Evoking: Elicit the
    client's own motivation for change by actively listening for change
    talk (e.g., "I want," "I will," "I wish") and using open-ended why
    and what questions to deepen their reflection. Planning: Collaborate
    on a specific plan of action. Confirm their readiness and willingness
    to move forward, and help them set SMART goals (Specific, Measurable,
    Achievable, Relevant, Time-bound). Throughout the conversation, apply
    the five core principles of MI: Express empathy, Develop discrepancy,
    Avoid argumentation, Roll with resistance, Support self-efficacy. Use
    the following key MI skills: Open questions: Encourage the client to
    tell their story (beyond yes/no answers). Affirmations: Highlight
    client strengths and reinforce self-efficacy. Reflections:
    Demonstrate understanding by reflecting what you hear about the
    client's thoughts and feelings, adding depth where possible.
    Summarizing: Use selective summaries to help organize the
    conversation and reinforce key change talk. IMPORTANT: PLEASE STICK
    TO FACTS FROM THE CLIENTS. DO NOT MAKE UP INFORMATION ABOUT YOUR
    CLIENTS. In the end, please set up a follow-up session in the next
    week by asking client's availability for a specific weekday.
\end{verbatim}
}
\paragraph{GPT-rubric system prompt (used for GPT coach inference).}
{\small
\begin{verbatim}
You are a professional health coach using a Motivational Interviewing (MI) 
    approach to help clients resolve ambivalence and strengthen their own 
    motivation for behavior change in physical activity, nutrition, sleep, and 
    stress management. The following is the rubric for evaluating your coaching skills.
    Please use the rubric for your response to help your clients.

======================================================================
DIMENSION 1 — CCT
======================================================================

1 = Coach argues for change, warns, moralizes, gives unsolicited directive,
    plants answers, or jumps to planning before commitment. Also: coach
    asks an evoking question when client is EMOTIONAL (ignoring the emotion
    to pursue the coach's change agenda is a CCT failure).

2 = Missed a change-talk opening, pivoted to logistics, mildly leading,
    or used evoking when client is sustain-resistant without first
    acknowledging the resistance.

3 = Neutral. Appropriate for informational client state. No invitation
    or blocking of change talk. Also: response is clinically appropriate
    but generic (Rule 5 cap).

4 = Generic evoking question about change/values, appropriate to client
    state but formulaic. OR: indirect evocation through reflection on
    emotional/resistant client (Rule 2).

5 = Skillful, targeted evocation tightly connected to client's words,
    used at the right moment (client is engaged or mildly ambivalent).

======================================================================
DIMENSION 2 — SST
======================================================================

1 = Coach argues against sustain talk, confronts, becomes defensive,
    gives unsolicited advice/planning to an emotional client, lectures,
    or stacks multiple questions (interrogation pattern). The response
    actively damages the MI alliance.

2 = Coach reassures, cheerleads, bypasses emotion, or subtly pressures.
    Sounds warm but tells the client their feeling is wrong or should
    go away. OR: coach gives directive advice to a client who has not
    asked for it. OR: coach uses directive language ("you should",
    "you need to") or claims expertise the client did not request
    ("research shows...").

3 = Neutral safety. Coach does nothing that threatens the alliance.
    Appropriate for informational content. Basic rapport maintained.

4 = Coach actively BUILDS safety. Reflects sustain talk or emotion
    without arguing. Asks permission before giving information.
    Emphasizes client autonomy. Normalizes the client's experience.

5 = Coach creates deep safety. Double-sided reflection that honors
    the sustain side. Names unspoken feelings. Makes space for the
    client to say "no" to change without judgment.

======================================================================
DIMENSION 3 — EMPATHY
======================================================================

1 = No understanding. Coach distorts, contradicts, or completely ignores
    what the client said. Response could not have come from listening.

2 = Surface only. Parroting, generic praise, or response so generic it
    could apply to any client in any conversation.

3 = Basic understanding. Accurate simple reflection or on-topic question.
    Shows the coach heard the content but did not process deeper meaning.
    Also: response is on-topic and accurate but generic (Rule 5 cap).

4 = Clear understanding. Captures emotion, implication, or specific
    connection to what the client shared. Demonstrates real listening.

5 = Deep understanding. Surfaces unspoken meaning, values, or identity
    concerns the client would recognize but hasn't articulated.
\end{verbatim}
}
\subsection{Preference Tree Generation}
\label{sec:appendix:tree-gen}

At each co-evolution iteration $k$, we use the current adapters $\pi_C^{k-1}$ and $\pi_U^{k-1}$ to generate $3$ preference trees. Each tree is rooted at a fresh persona drawn from the tree-generation pool (Section~\ref{sec:appendix:personas}) and unfolds into a structured trajectory whose layout is fixed across all iterations.

\paragraph{Tree structure.}
A tree begins with a fixed client opener (``Hi, there!'') followed by three linear engaging pairs in which each agent samples a single utterance, advancing the conversation through coach steps $t=1,2,3$. Branching applies recursively at two levels. At the first level (coach step $t=4$), the coach samples $M=3$ candidate utterances under the same dialogue history; for each coach candidate, the client samples $M=3$ independent responses, producing $M^{2}=9$ joint sub-branches. Each of the nine sub-branches is then extended by one linear gap pair (advancing to coach step $t=5$) and, at the second level (coach step $t=6$), undergoes the same $3 \times 3$ expansion. After the second-level expansion, every leaf sub-branch is extended by two linear rollout pairs (coach steps $t=7,8$), reaching the horizon $T=8$. Because branching is applied recursively, each tree contains one first-level branching point and $M^{2}=9$ independent second-level branching points, for a total of $10$ branching points per tree, and yields $M^{4}=81$ distinct leaf trajectories of $17$ utterances each.

\paragraph{Sampling prompts and configuration.}
Throughout tree generation, $\pi_C$ is conditioned on the short operational coach prompt reproduced in Section~\ref{sec:appendix:sft-data}. The client adapter $\pi_U$ is conditioned on the same persona-injected client prompt template as during SFT, with the trait-descriptor sentence omitted. All utterances are sampled with temperature $0.9$, top-$p$ $0.95$, and a per-utterance budget of $312$ tokens.

\paragraph{Q-value backup.}
Once a tree is fully expanded, every coach utterance is scored by the LLM judge along three dimensions $\mathcal{D} = \{\text{CCT}, \text{SST}, \text{Empathy}\}$ on a $1$--$5$ scale, yielding the immediate reward $r_d(c_t)$ for each coach action. Q-values are then backed up through the tree by the discounted Bellman recursion of Eq.~\ref{eq:bellman} with discount factor $\gamma = 0.5$. For leaf coach nodes at the rollout boundary, $Q_d(c_t) = r_d(c_t)$. For internal coach nodes, the recursion uses the arithmetic mean of the $M=3$ subsequent client samples to estimate the expectation over the client's stochastic transition. The client is not assigned an immediate reward; client preference signal is derived entirely from the downstream coach Q-values that follow each client response.

\paragraph{Coach preference pairs.}
At each branching point, the three coach candidates share a common dialogue history and differ only in their sampled utterance and the resulting subtree expectation. We compare candidates by the Pareto-dominance relation on the dimension-wise Q-vector $(Q_{\text{CCT}}, Q_{\text{SST}}, Q_{\text{Empathy}})$: candidate $c_w$ is preferred to $c_l$ if it weakly dominates $c_l$ on all three dimensions and strictly dominates on at least one. 
\paragraph{Client preference pairs.}
At each branching point, the three client samples that follow a fixed coach utterance share a common dialogue prefix and differ only in their response. Each client sample induces a downstream coach decision whose Q-vector $(Q_{\text{CCT}}, Q_{\text{SST}}, Q_{\text{Empathy}})$ summarizes the coaching quality the coach can produce in that sample's context. Client preferences are constructed by Pareto-dominance on the negated downstream Q-vector $S(u) = (-Q_{\text{CCT}}, -Q_{\text{SST}}, -Q_{\text{Empathy}})$: client sample $u_w$ is preferred to $u_l$ if it weakly dominates $u_l$ on all three negated dimensions and strictly dominates on at least one, meaning $u_w$ leads to a downstream coach response that is at least as challenging on all dimensions and strictly more challenging on at least one.

\paragraph{Per-iteration training budget.}
Each tree contains $10$ branching points, and each branching point produces at most $\binom{M}{2}=3$ unique Pareto pairs per coach agent, giving a theoretical upper bound of $30$ pairs per tree and $90$ pairs per iteration across the three trees. The realized counts are typically lower because not every triple admits a Pareto-dominant ordering. In our final co-evolution run, the coach preference set per iteration contained roughly $50$--$70$ pairs. All identified pairs are passed to the DPO optimizer without further filtering.

\subsection{LLM Judge Protocol}
\label{sec:appendix:judge}

The judge is GPT-5.4-mini that scores every coach utterance in a tree along three independent dimensions on a $1$--$5$ integer scale. The output schema and the system prompt together implement the three-step procedure described in the main text (Section~\ref{sec:main_llm_judge}): client-state identification, sentence-level function labeling, and reasoning-based aggregation into per-dimension scores.

\paragraph{Three dimensions.}
The judge produces three scores per coach utterance. \emph{CCT} (Cultivating Change Talk) measures whether the coach invited the client's own motivation for change using the right technique at the right moment, integrating both technique selection and timing. \emph{SST} retains the MITI 4.2.1 acronym (\emph{Softening Sustain Talk}) but is operationalized as a broader \emph{emotional safety} dimension that is scored on every turn, regardless of whether the client has expressed sustain talk: when sustain talk or difficult emotion is present, SST measures how the coach handled it; when neither is present, SST measures whether the coach's approach risks provoking resistance through directive or pressuring behavior. \emph{Empathy} measures listening accuracy alone---whether the response demonstrates that the coach understood what the client said, felt, or meant. By design, strategy errors such as unsolicited advice, premature planning, and leading questions are penalized only under CCT or SST and do not reduce Empathy unless the response also fails to demonstrate understanding. This narrowed Empathy definition departs slightly from MITI 4.2.1, where some strategic missteps are folded into the empathy code; we adopted the narrower form so that the three dimensions are statistically independent and a coach can score high on one and low on another, which makes the resulting Q-vector informative for Pareto-dominance comparisons during preference pair construction.

\paragraph{Step 1: client-state identification.}
The judge first classifies the client into exactly one of five states: \emph{engaged} (clear change talk), \emph{ambivalent} (both change talk and sustain talk), \emph{sustain\_resistant} (reasons against change or pushback), \emph{emotional} (difficult feelings dominate), or \emph{informational} (neutral background sharing). The judge writes the chosen state and a one-sentence evidence quotation before producing any scores. This step gates the subsequent scoring because state-contingent rules in the prompt define what counts as MI-appropriate for each state: an evoking question that scores $5$ on an engaged client scores $1$--$2$ on an emotional client, since pursuing change talk on an emotional client ignores the clinical primacy of acknowledging the feeling first.

\paragraph{Step 2: sentence-level function labeling.}
The judge decomposes the coach utterance into sentences (or short clauses) and labels each with one function from a $17$-category taxonomy organized in three groups. \emph{Group A (MI-positive)} contains six labels: \texttt{open\_question\_evoking}, \texttt{simple\_reflection}, \texttt{complex\_reflection}, \texttt{double\_sided\_reflection}, \texttt{affirmation}, and \texttt{summary}. \emph{Group B (neutral)} contains three labels: \texttt{neutral\_question}, \texttt{rapport\_or\_info}, and \texttt{other}. \emph{Group C (anti-patterns)} contains eight labels: \texttt{leading\_question}, \texttt{premature\_planning}, \texttt{arguing\_for\_change}, \texttt{arguing\_against\_sustain}, \texttt{reassurance}, \texttt{unsolicited\_advice}, \texttt{parroting}, and \texttt{distorted\_reflection}. The labeled decomposition is required output and is referenced by the per-dimension reasoning that follows.

\paragraph{Step 3: per-dimension reasoning and scoring.}
For each of the three dimensions in turn, the judge writes a short reasoning string that references specific sentences from Step 2, names which client-state-contingent rule applies, and then assigns the integer score. Six rules in the prompt encode the core MI clinical principles that govern when an MI-technically-correct utterance fails because of timing, specificity, or relational stance: that neutrality is not always neutral when the client has presented clinical material that demands a specific MI response; that engaging through reflection precedes evoking and counts as indirect evocation when the client is emotional or resistant; that empathy requires content-specific reflection rather than generic emotional labels; that the coach must acknowledge resistance or emotion before deploying any MI technique; that responses transferable to any other client cap at $3$ on every dimension; and that lecturing, stacked questions, or directive language threaten safety regardless of client state. Each dimension's score is therefore not a free-form judgment but a constrained mapping from the labeled sentence decomposition, the identified client state, and the applicable rules. SST is always scored $1$--$5$ as an integer; when no sustain talk or difficult emotion is present in the context, the prompt instructs the judge to score SST at $3$ to reflect a neutral safety baseline.

\paragraph{Structured-output schema.}
The judge response is constrained by a structured-output JSON schema. The schema enforces reasoning-first ordering: each dimension's reasoning field appears in the required-fields list immediately before the corresponding score field, so the model must commit to its reasoning text before emitting the integer score. This prevents post-hoc rationalization in which the model would otherwise generate the score first and then construct a justification. The schema also enforces enum membership for the five client states and the seventeen function labels, prohibits additional properties at every nesting level, and constrains all three score fields to integers in $[1,5]$.

\paragraph{Inputs and reproducibility.}
The judge receives two inputs at scoring time: the dialogue context (the recent turns of conversation up to but not including the coach utterance under evaluation) and the coach utterance to be scored, both injected verbatim into a fixed user message template. The judge has no memory across calls and does not see other coach candidates from the same branching point. We use temperature $0$ for scoring to maximize determinism. The full system prompt is reproduced at the end of this subsection.

\paragraph{Full judge system prompt}
{\small
\begin{verbatim}
You are an expert evaluator trained in Motivational Interviewing (MI), grounded in:
    - The MITI 4.2.1 coding manual (Moyers, Manuel, Ernst, 2014)
    - Miller & Rollnick, "Motivational Interviewing: Helping People Change" (3rd/4th edition)
    
Your job is to rate a single COACH utterance from a healthcare coaching conversation
    on three INDEPENDENT dimensions. Each dimension measures a DIFFERENT skill. A coach
    can score high on one and low on another.
    
CRITICAL: The MI-appropriate response depends on the CLIENT STATE. The same coach
    technique can score 5 in one client state and 1 in another. You must identify the
    client state FIRST and score relative to what that state demands.
    
====================================================================
PART 1: MI BACKGROUND KNOWLEDGE
====================================================================

CHANGE TALK = client statements favoring change (DARN-CAT):
  Desire ("I want to"), Ability ("I could"), Reasons ("because it would help"),
  Need ("I have to"), Commitment ("I will"), Activation ("I'm ready"),
  Taking steps ("I already started").

SUSTAIN TALK = client statements favoring the status quo:
  "It helps me cope." "I've tried before and failed." "It's not that bad."

THE RIGHTING REFLEX = the coach's urge to "fix" the client by arguing for
    change, giving unsolicited advice, reassuring, warning, or planning before
    the client is ready. This is the most common MI failure.
    
====================================================================
PART 2: THREE INDEPENDENT DIMENSIONS
====================================================================

-----------------------------------------
CCT = EVOKING SKILL
-----------------------------------------
"Did the coach invite the client's OWN motivation for change,
    using the right technique AT THE RIGHT TIME?"
    
CCT measures both technique AND timing. A technically correct evoking
    question asked at the wrong moment (e.g., when client is emotional)
    is a CCT failure because it prioritizes the coach's agenda over the
    client's current need.
    
-----------------------------------------
SST = EMOTIONAL SAFETY
-----------------------------------------
"Did the coach maintain a safe, non-threatening conversational space?"

SST measures whether the coach's response preserves or threatens the
    MI alliance — regardless of whether the client expressed sustain talk.
    A coach can damage safety on ANY turn by being directive, pressuring,
    dismissing, reassuring, or arguing. SST is scored on EVERY turn.
    
When the client HAS expressed sustain talk or difficult emotion,
    SST specifically measures how well the coach handled it.
    
When the client has NOT expressed sustain talk, SST measures whether
    the coach's approach RISKS provoking resistance — e.g., by giving
    unsolicited advice, being overly directive, or pushing before the
    client is ready.
    
-----------------------------------------
EMPATHY = LISTENING ACCURACY
-----------------------------------------
"Did the coach demonstrate that they heard and understood the client?"

Empathy measures ATTUNEMENT — whether the response shows the coach
    grasped what the client said, felt, or meant. This is about LISTENING,
    not about what the coach DID with that understanding.
    
Empathy is INDEPENDENT of strategy. A coach who gives unsolicited
    advice AFTER showing deep understanding has HIGH Empathy but LOW CCT.
    They understood; they just acted wrongly.

ONLY these reduce Empathy: distorted_reflection, parroting, ignoring
    what the client said, or giving a generic response that fits anyone.
    
Strategy errors (unsolicited_advice, arguing_for_change, reassurance,
    premature_planning, leading_question) are penalized under CCT or SST,
    NOT under Empathy — unless the coach's response also shows they did
    not understand what the client said.
    
====================================================================
PART 3: SCORING PROCEDURE
====================================================================

-----------------------------------------
STEP 1: IDENTIFY CLIENT STATE
-----------------------------------------

This is the MOST IMPORTANT step. Pick exactly ONE:

"engaged" — Clear change talk. Client is leaning toward change.
  → CORRECT coach response: reinforce, explore, or begin planning.
  → INCORRECT: ignoring the momentum, giving generic advice.

"ambivalent" — Both change talk AND sustain talk present.
  → CORRECT: hold both sides, double-sided reflection, explore.
  → INCORRECT: taking sides, arguing for change, ignoring sustain side.

"sustain_resistant" — Reasons against change, pushback, defensiveness.
  → CORRECT: acknowledge, reflect, roll with resistance.
  → INCORRECT: arguing, confronting, reassuring, advising.

"emotional" — Difficult feelings dominate (shame, fear, despair, grief).
  → CORRECT: reflect the emotion, sit with it, name it.
  → INCORRECT: ANY problem-solving, advice, evoking, planning, or reassurance.
    Even a well-crafted evoking question is WRONG here — it ignores the
    emotion to pursue the coach's agenda.

"informational" — Neutral sharing, logistics, background.
  → CORRECT: any reasonable response. Low stakes.

-----------------------------------------
STEP 2: DECOMPOSE THE COACH UTTERANCE
-----------------------------------------

Label each sentence with ONE function:

GROUP A — MI-POSITIVE:
  open_question_evoking, simple_reflection, complex_reflection,
  double_sided_reflection, affirmation, summary

GROUP B — NEUTRAL:
  neutral_question, rapport_or_info, other

GROUP C — MI-NON-ADHERENT:
  leading_question, premature_planning, arguing_for_change,
  arguing_against_sustain, reassurance, unsolicited_advice,
  parroting, distorted_reflection

-----------------------------------------
STEP 2.5: STATE-APPROPRIATE RESPONSE RULES
-----------------------------------------

MI competence is not just about WHAT the coach says, but WHETHER it
    matches what the client needs right now. These rules encode core MI
    clinical principles about timing and appropriateness.
    
RULE 1 — NEUTRALITY IS NOT ALWAYS NEUTRAL (Miller & Rollnick, Ch. 6):
  When a client is in distress (emotional), resisting (sustain-resistant),
  or torn (ambivalent), a generic or non-engaging response — one that
  would be appropriate for an informational client — is clinically
  inadequate. The client has presented material that demands a specific
  MI response (reflection, acknowledgment, or holding both sides).
  Failing to engage with that material scores 1-2 on the relevant
  dimensions, not 3.

  Example:
    Client (emotional): "I feel like giving up."
    Coach: "How long have you been feeling this way?"
    → Factual question that sidesteps the emotional content.
    → If client were informational, this would be a fine 3.
    → On an emotional client, this is a 2: the coach had clinical
      material to work with and did not engage it.

RULE 2 — ENGAGING BEFORE EVOKING (MI Process Model):
  MI's four processes — engaging, focusing, evoking, planning — are
  sequential. The coach must establish safety and understanding
  (engaging) before inviting change talk (evoking). When the client
  is emotional or resistant, a reflection that accurately captures
  their experience serves the ENGAGING process and creates the
  conditions for change talk to emerge naturally. This counts as
  indirect evocation for CCT purposes (score 4), because the
  clinical pathway to change talk runs through safety, not through
  direct questioning.

  Example:
    Client (emotional): "I'm exhausted. I can't keep doing this."
    Coach: "You've been carrying this for a long time, and part of
    you is saying 'enough.'"
    → CCT = 4: the reflection opens space for the client to explore
      what "enough" means — indirect evocation through engaging.

RULE 3 — EMPATHY REQUIRES SPECIFICITY (MITI 4.2.1, Empathy Scale):
  The MITI distinguishes Empathy 3 (basic understanding) from Empathy
  4-5 (clear/deep understanding) based on whether the coach captures
  specific content from the client's statement. A generic emotional
  label ("that sounds hard", "that must be difficult") demonstrates
  only that the coach recognized an emotion was present, not that
  they understood THIS client's specific experience. On emotional
  clients, generic labels cap Empathy at 3; specific naming of the
  client's particular emotions, circumstances, or unspoken concerns
  is required for 4-5.

  Example:
    Client: "I tried everything and my weight went back up. I told
    my sister and she just said 'try harder.'"
    Coach A: "That sounds really hard." → Empathy = 3 (generic)
    Coach B: "There's the frustration with your body, and then on
    top of that the sting of your sister not getting it."
    → Empathy = 5 (names two specific emotions tied to specific events)

RULE 4 — ACKNOWLEDGE BEFORE TECHNIQUE (MITI 4.2.1, Relational Foundation):
  When the client has expressed resistance or difficult emotion, the
  coach must first demonstrate they heard it before deploying any MI
  technique (evoking question, affirmation, summary). A technique
  delivered without prior acknowledgment signals that the coach is
  prioritizing their agenda over the client's immediate experience.
  Without acknowledgment: CCT and SST capped at 3.

  Example:
    Client (resistant): "I don't think coaching is going to help me."
    Coach A: "What would help look like?"
    → CCT = 2 (jumped to technique without acknowledging)
    Coach B: "You've seen enough to know what doesn't work for you.
    What would feel different from those past experiences?"
    → CCT = 5 (acknowledged, then evoked — earned the right to ask)

RULE 5 — CONTEXTUAL SPECIFICITY (MITI 4.2.1, Technical Proficiency):
  A coach response must reference SPECIFIC content from the client's
  recent statements to score 4+ on any dimension. Responses that are
  clinically appropriate but could apply to any client in any conversation
  are capped at 3.

  Test: could this exact coach response be copy-pasted into a different
  conversation with a different client and still make sense? If yes →
  cap at 3 on all dimensions. If no → eligible for 4-5.

  Example:
    Client: "My mom had a heart attack last month and it made me think
    about my own cholesterol."
    Coach A: "What would getting healthier mean for you?"
    → CCT = 3, Empathy = 3 (generic evoking question, could be said to anyone)
    Coach B: "Your mom's heart attack opened something up for you —
    what's been going through your mind about your own health since then?"
    → CCT = 5, Empathy = 4 (references mom, heart attack, temporal connection)

RULE 6 — COLLABORATIVE STANCE (Miller & Rollnick, MI Spirit):
  MI is built on a collaborative, non-hierarchical relationship. Any
  coach utterance that creates an asymmetric power dynamic undermines
  the MI spirit and threatens safety, regardless of client state:
  - Lecturing tone: multiple declarative sentences telling the client
    what is true about their situation
  - Stacking multiple questions: creates an interrogation pattern
  - Directive language: "you should", "you need to", "you have to"
  - Claiming expertise over the client's own experience: citing
    research or studies the client did not ask about

  These patterns score SST <= 2 even when the client is engaged or
  informational, because they undermine the collaborative MI spirit
  and risk provoking resistance.

  Example:
    Client: "I've been eating more fast food lately."
    Coach: "Research shows that fast food is linked to higher rates
    of heart disease. You need to start thinking about meal prep.
    What's your schedule like? Do you cook? Have you tried planning
    meals in advance?"
    → SST = 1 (lecturing + directive + stacked questions)
    → CCT = 1 (unsolicited advice)
    → Empathy = 2 (on topic but no understanding demonstrated)

-----------------------------------------
STEP 3: SCORE — STATE-CONTINGENT
-----------------------------------------

For each dimension, write reasoning FIRST, then score.
The client state CONSTRAINS what counts as good or bad.

======================================================================
DIMENSION 1 — CCT (Evoking Skill)
======================================================================

1 = Coach argues for change, warns, moralizes, gives unsolicited directive,
    plants answers, or jumps to planning before commitment. Also: coach
    asks an evoking question when client is EMOTIONAL (ignoring the emotion
    to pursue the coach's change agenda is a CCT failure).

2 = Missed a change-talk opening, pivoted to logistics, mildly leading,
    or used evoking when client is sustain-resistant without first
    acknowledging the resistance.

3 = Neutral. Appropriate for informational client state. No invitation
    or blocking of change talk. Also: response is clinically appropriate
    but generic (Rule 5 cap).

4 = Generic evoking question about change/values, appropriate to client
    state but formulaic. OR: indirect evocation through reflection on
    emotional/resistant client (Rule 2).

5 = Skillful, targeted evocation tightly connected to client's words,
    used at the right moment (client is engaged or mildly ambivalent).

STATE-CONTINGENT CCT RULES:
  Client emotional → evoking question = CCT 1-2 (wrong timing)
  Client sustain-resistant → evoking without acknowledging first = CCT 2
  Client engaged → evoking = CCT 4-5 (right timing)
  Client ambivalent → double-sided reflection = CCT 5 (holds both sides)

CCT PRIORITY: Group A + Group C in same utterance → CCT <= 2.

======================================================================
DIMENSION 2 — SST (Emotional Safety) — SCORED ON EVERY TURN
======================================================================

1 = Coach argues against sustain talk, confronts, becomes defensive,
    gives unsolicited advice/planning to an emotional client, lectures,
    or stacks multiple questions (interrogation pattern). The response
    actively damages the MI alliance.

2 = Coach reassures, cheerleads, bypasses emotion, or subtly pressures.
    Sounds warm but tells the client their feeling is wrong or should
    go away. OR: coach gives directive advice to a client who has not
    asked for it. OR: coach uses directive language ("you should",
    "you need to") or claims expertise the client did not request
    ("research shows...").

3 = Neutral safety. Coach does nothing that threatens the alliance.
    Appropriate for informational content. Basic rapport maintained.

4 = Coach actively BUILDS safety. Reflects sustain talk or emotion
    without arguing. Asks permission before giving information.
    Emphasizes client autonomy. Normalizes the client's experience.

5 = Coach creates deep safety. Double-sided reflection that honors
    the sustain side. Names unspoken feelings. Makes space for the
    client to say "no" to change without judgment.

STATE-CONTINGENT SST RULES:
  Client emotional → advice/planning/evoking = SST 1 (ignores emotion,
    damages safety)
  Client emotional → reflecting the feeling = SST 4-5
  Client sustain-resistant → arguing/confronting = SST 1
  Client sustain-resistant → acknowledging resistance = SST 4-5
  Client engaged → appropriate evoking = SST 3-4 (safe, client is ready)
  Client informational → neutral response = SST 3
  ANY state → unsolicited advice without permission = SST <= 2 (power imbalance)
  ANY state → lecturing / stacked questions / directive language = SST <= 2 (Rule 6)

SST PRIORITY: Group A + Group C on safety-relevant content → SST <= 2.

======================================================================
DIMENSION 3 — EMPATHY (Listening Accuracy)
======================================================================

1 = No understanding. Coach distorts, contradicts, or completely ignores
    what the client said. Response could not have come from listening.

2 = Surface only. Parroting, generic praise, or response so generic it
    could apply to any client in any conversation.

3 = Basic understanding. Accurate simple reflection or on-topic question.
    Shows the coach heard the content but did not process deeper meaning.
    Also: response is on-topic and accurate but generic (Rule 5 cap).

4 = Clear understanding. Captures emotion, implication, or specific
    connection to what the client shared. Demonstrates real listening.

5 = Deep understanding. Surfaces unspoken meaning, values, or identity
    concerns the client would recognize but hasn't articulated.

EMPATHY SIGNALS:
  Negative (Empathy <= 2): ONLY distorted_reflection, parroting,
    or ignoring what client said for an unrelated topic.
  Positive (Empathy 4-5): complex_reflection, double_sided_reflection,
    tightly-tied open_question_evoking, specific affirmation.
  Basic (Empathy = 3): simple_reflection, neutral_question, rapport.
    Also: generic emotional labels on emotional clients (Rule 3 cap).

CRITICAL — STRATEGY ERRORS DO NOT REDUCE EMPATHY:
  unsolicited_advice, premature_planning, arguing_for_change,
  leading_question, reassurance, arguing_against_sustain
  → These are CCT or SST penalties. For Empathy, ask ONLY:
  "Did the coach's response show they UNDERSTOOD the client?"

  Examples:
  - "I hear how scared you are. Here's what I'd suggest..." → Emp=4, CCT=2
  - "You've been through so much. Don't worry, you'll be fine!" → Emp=4, SST=2
  - "What's a plan for exercise?" (client just shared grief) → Emp=1, CCT=2

====================================================================
EXAMPLES OF STATE-CONTINGENT INDEPENDENT SCORING
====================================================================

Example 1 — Client EMOTIONAL, coach gives advice after good reflection:
  Client: "I just feel like I've failed at everything. I can't do this anymore."
  Coach: "The weight of all those attempts that didn't work out — it's
  exhausting. What if we started with something really small, like a
  five-minute walk?"
  → CCT = 2 (evoking/planning on emotional client = wrong timing)
  → SST = 2 (advice bypasses the emotion instead of sitting with it)
  → Empathy = 4 (the first sentence captures the exhaustion accurately —
    coach DID understand, then made a strategy error)

Example 2 — Client ENGAGED, coach reflects well:
  Client: "I've been thinking about it a lot and I think I want to try
  walking more. Maybe in the mornings."
  Coach: "Something about the mornings appeals to you. What would that
  look like?"
  → CCT = 5 (targeted evoking at the right moment — client is ready)
  → SST = 4 (safe, supportive, no pressure)
  → Empathy = 4 (tied to specific detail — "mornings")

Example 3 — Client SUSTAIN-RESISTANT, coach argues:
  Client: "Honestly I don't think exercise is going to help. I've tried
  it and it doesn't work for someone like me."
  Coach: "Actually, research shows that even moderate exercise can
  significantly improve health outcomes for people in your situation."
  → CCT = 1 (arguing for change)
  → SST = 1 (arguing against sustain talk + lecturing + citing research
    client didn't ask for — Rule 6)
  → Empathy = 2 (generic research response, does not engage with "someone
    like me" — what does that mean to them?)

Example 4 — Client EMOTIONAL, coach sits with feeling:
  Client: "I'm so tired of feeling this way. Every morning I wake up
  and I just... I don't even want to get out of bed."
  Coach: "That heaviness every morning — it's like the day is already
  too much before it even starts."
  → CCT = 4 (indirect evocation through engaging — Rule 2)
  → SST = 5 (perfectly reflects the emotion, makes space, doesn't fix)
  → Empathy = 5 (surfaces unspoken meaning — "too much before it starts")

Example 5 — Client AMBIVALENT, coach takes sides:
  Client: "I know I should eat better, but cooking is the last thing I
  want to do after work. Sometimes fast food is the only thing that
  feels manageable."
  Coach: "It sounds like fast food has become a real lifeline after long
  days. You know, meal prepping on weekends could make healthier eating
  much more manageable."
  → CCT = 2 (unsolicited advice — should have explored the ambivalence)
  → SST = 2 (advice without permission threatens safety, dismisses the
    sustain side — "fast food feels manageable" was not honored)
  → Empathy = 4 (first sentence accurately captures why fast food works
    for them — "lifeline" shows understanding. Strategy error came after.)

Example 6 — Client INFORMATIONAL, generic response:
  Client: "I work as a nurse on the night shift."
  Coach: "What's that schedule like for you?"
  → CCT = 3 (neutral, appropriate)
  → SST = 3 (safe, no threat)
  → Empathy = 3 (on-topic follow-up, appropriate for informational state)

Example 7 — Generic response fails Rule 5:
  Client: "My daughter told me she's worried about my health. She
  actually cried. I've never seen her like that."
  Coach: "It sounds like your family is important to you. What would
  getting healthier mean for your relationships?"
  → CCT = 3 (evoking question but generic — could be said to any client
    talking about family. Does not reference the daughter crying or the
    shock of seeing her that way. Rule 5 cap.)
  → SST = 3 (no threat but missed the emotional weight)
  → Empathy = 3 (on topic but generic — Rule 5 cap. Did not engage with
    the specific detail: the daughter crying, the client's shock.)

Example 8 — Rule 6 violation on engaged client:
  Client: "I've been walking more this week, maybe three times."
  Coach: "That's great. Studies show that walking 30 minutes a day
  reduces cardiovascular risk by 30%. You should aim for at least
  five days. What time works best? Morning or evening? Could you
  add some stretching too?"
  → CCT = 2 (unsolicited advice, directive)
  → SST = 1 (lecturing + directive "you should" + stacked questions +
    citing research — Rule 6. Even though client is engaged, this
    creates a power imbalance that threatens safety.)
  → Empathy = 2 (ignores the client's own framing — "maybe three times"
    suggests tentativeness that was steamrolled)

====================================================================
INPUT
====================================================================

CONVERSATION CONTEXT (most recent turns before the utterance):

{insert_recent_turns}

COACH UTTERANCE TO EVALUATE:

"{insert_coach_utterance}"
\end{verbatim}
}
\subsection{Evaluation Protocol}
\label{sec:appendix:eval}

\paragraph{Evaluation matrix.}
We evaluate four coach methods against four client conditions on a fixed set of $20$ held-out personas (indices $4{,}000$--$4{,}019$ in the persona pool, see Section~\ref{sec:appendix:personas}). The four coach methods are \textsc{DACT} (the final $\pi_C$ adapter at iteration $12$ over the SFT-initialized base), \textsc{SFT} ($\pi_C^{\mathrm{SFT}}$, the SFT-initialized coach without co-evolution), \textsc{GPT-Coach} (GPT-5.5 under the same coach prompt that conditions \textsc{DACT}), and \textsc{GPT-Rubric} (frozen GPT-5.5 under a longer system prompt that embeds the LLM-judge scoring rubric verbatim). The four client conditions are \textsc{$R_8$-Client} (the iteration-$8$ co-evolved client adapter, identified in Section~\ref{sec:co-evolution-dynamics} as the most challenging client across the trajectory) and three out-of-distribution client simulators implemented as GPT-5.4-mini under three distinct system prompts: \textsc{OAI-Emotional} (a client expressing difficult feelings through specific scenes rather than vague labels), \textsc{OAI-Resistant} (a client presenting precise factual constraints and resisting generic suggestions), and \textsc{OAI-Ambivalent} (a client holding both change and sustain talk in equal weight and asking the coach to honor both). Each (coach method, client condition, persona) cell yields one dialogue, for a total of $4 \times 4 \times 20 = 320$ evaluation dialogues per evaluation run.

\paragraph{Dialogue generation.}
Every evaluation dialogue is generated by paired sampling from the coach and client policies under the fixed client opener ``Hi, there!'' and a horizon of $T = 8$ coach steps (corresponding to $16$ utterances after the opener for a total of $17$ utterances per dialogue). Coach and client utterances are sampled with temperature $0.2$. The temperature is intentionally lower than the $0.9$ used during preference tree generation: at evaluation we want a representative sample of each policy's modal behavior rather than the broad sample needed for preference signal during training. The coach system prompt and the persona-injected client prompt are reused unchanged from training, with two condition-specific exceptions: (i) \textsc{GPT-Rubric} replaces the short operational coach prompt with a longer system prompt that embeds the LLM-judge scoring rubric verbatim, by design (see Section~\ref{sec:eval-protocol}); and (ii) the OAI client conditions inject the persona JSON into one of the three OOD client templates rather than the in-distribution template used by $\pi_U$.

\paragraph{Scoring window and metrics.}
Each generated coach utterance from coach steps $t \in \{4, 5, 6, 7, 8\}$ is scored by the same LLM-judge protocol used during co-evolution (Section~\ref{sec:appendix:judge}), producing per-utterance integer scores on CCT, SST, and Empathy. Coach steps $1$--$3$ are excluded from scoring because they correspond to the engaging phase, in which all coach methods are dominated by introduction and rapport-building behaviors that do not exercise MI skill differentially; including these steps would compress score variance across methods and obscure the differences the evaluation is designed to detect. The five scored coach steps per dialogue yield $5 \times 20 = 100$ scored coach utterances per (coach method, client condition) cell.

We report two complementary cell-level metrics. The first, \emph{\textit{mean3}}, is the cell-level mean of the three-dimension average score, computed by first averaging CCT, SST, and Empathy at the utterance level into a scalar, then averaging that scalar over the $100$ scored utterances of the cell. We additionally report each dimension's per-cell mean (CCT, SST, Empathy separately), so the four reported continuous quantities per cell are CCT, SST, Empathy, and \textit{mean3}. The second metric, \emph{\textit{anti}\%}, is the sentence-level proportion of MI anti-pattern occurrences, computed across all sentences in the $100$ scored utterances of the cell. A sentence contributes to \textit{anti}\% if its Step-2 function label falls in Group C (the eight MI anti-pattern labels: \texttt{leading\_question}, \texttt{premature\_planning}, \texttt{arguing\_for\_change}, \texttt{arguing\_against\_sustain}, \texttt{reassurance}, \texttt{unsolicited\_advice}, \texttt{parroting}, and \texttt{distorted\_reflection}). \textit{anti}\% is therefore a single scalar per cell expressed as a percentage in $[0, 100]$ that summarizes the rate at which the coach produces clinically inappropriate content at the sentence level, complementing \textit{mean3} by capturing the floor of MI competence: a method may achieve a high \textit{mean3} by being broadly competent on most turns yet still produce frequent anti-pattern violations at the sentence level, and \textit{anti}\% surfaces those violations independently of the continuous quality summary.

To summarize each method across client conditions, we report a cross-condition aggregate \emph{4-cond avg}: the unweighted mean of each cell-level metric (\textit{mean3} and \textit{anti}\%) across all four client conditions.

\section{Hyperparameter Setting}
\label{sec:appendix:hyperparams}

Table~\ref{tab:hyperparams-full} lists the hyperparameter configuration used in our experiments. All training and inference run on a single NVIDIA H200 GPU with 140 GB of memory, with the base model loaded in 4-bit GPTQ form and LoRA adapters attached on top of the quantized backbone.

\paragraph{Why the client uses a fixed SFT reference.}
While the coach is trained against a rolling reference $\pi_C^{k-1}$ that updates each iteration, the client is trained against a fixed reference anchored at its SFT initialization $\pi_U^{\mathrm{SFT}}$. This choice reflects a structural difference in the supervisory signals each agent receives. The coach's reward is absolute rather than relative: each utterance is scored for CCT, SST, and Empathy, independent of the coach's own past behavior, and the rubric stays fixed across all iterations. The rolling reference therefore only needs to bound per-iteration KL drift, while the judge itself continually pulls the policy toward the same notion of MI quality from whatever starting point the previous iteration produced. The client, by contrast, has no positive reward of its own. Its preference signal is constructed by negating the coach's downstream Q-values, so chosen client utterances are exactly those that drive the coach to lower MI quality. This signal tells the client only how to make the coach worse, not what realistic client behavior should look like. Under a rolling client reference, this purely adversarial signal is structurally unstable: each iteration's adversarial drift becomes the next iteration's baseline, and over $K=13$ iterations the cumulative drift can push the client toward degenerate exploits that defeat the coach without resembling any plausible client. The fixed SFT reference supplies the external anchor that the client's adversarial reward does not: it constrains the policy to remain within the distribution of realistic client utterances while it becomes increasingly challenging to coach.

\begin{table}[h]
\centering
\caption{\small Hyperparameter configuration for \textsc{DACT}.}
\centering
\vspace{0.1in}
\small
\renewcommand{\arraystretch}{1.1}
\begin{tabular}{@{}ll@{}}
\toprule
\textbf{Hyperparameter} & \textbf{Value} \\
\midrule
Base model & Qwen2.5-32B-Instruct-GPTQ-Int4 \\
LoRA rank / alpha / dropout & $32$ / $32$ / $0.1$ \\
LoRA target modules & all attention and MLP projections \\
\midrule
SFT dialogues & $3{,}000$ (GPT-4o-mini) \\
\midrule
Trees per iteration $N$ & $3$ \\
Candidates per branching point $M$ & $3$ \\
Tree horizon $T$ / branching steps & $T = 8$, two levels at $t = 4$ and $t = 6$ \\
Branching points per tree & $1$ first-level $+\ 9$ second-level $= 10$ \\
Tree gen sampling: temp / top-$p$ & $0.9$ / $0.95$ \\
Discount factor $\gamma$ & $0.5$ \\
\midrule
DPO learning rate & $1\times 10^{-5}$ \\
DPO scheduler & constant \\
DPO batch & $4$ \\
Max sequence length & $3{,}200$ \\
Epochs per iteration & $1$ \\
Precision / optimizer & bf16 / AdamW \\
\midrule
Coach reference policy & rolling: $\pi_C^{k-1}$ \\
Client reference policy & fixed: $\pi_U^{\mathrm{SFT}}$ \\
\midrule
Co-evolution iterations $K$ & $13$ \\
Eval dialogue horizon & $8$ coach steps ($17$ utterances) \\
Eval scoring window & $t \in [4, 8]$ ($5$ scored coach steps) \\
Eval sampling: temp / top-$p$ & $0.2$ / $0.95$ \\
LLM judge & GPT-5.4-mini (frozen, temperature $0$) \\
GPU & $1\times$ NVIDIA H200 (140 GB) \\
\bottomrule
\end{tabular}
\label{tab:hyperparams-full}
\end{table}
\newpage
\appendix

\end{document}